\documentclass[letterpaper]{article} 
\usepackage{aaai2026}  
\usepackage{times}  
\usepackage{helvet}  
\usepackage{courier}  
\usepackage[hyphens]{url}  
\usepackage{graphicx} 
\usepackage{natbib}  
\usepackage{caption} 
\usepackage{algorithm}
\usepackage{algorithmic}

\nocopyright

\usepackage{newfloat}
\usepackage{listings}
\DeclareCaptionStyle{ruled}{labelfont=normalfont,labelsep=colon,strut=off} 
\floatstyle{ruled}
\newfloat{listing}{tb}{lst}{}
\floatname{listing}{Listing}
\title{Towards Understanding the Expressive Power of GNNs with Global Readout}
\author{
    Maurice Funk\textsuperscript{\rm 1, \rm 2} and
    Daumantas Kojelis\textsuperscript{\rm 1}
}
\affiliations{
    \textsuperscript{\rm 1} Leipzig University, Germany\\
    \textsuperscript{\rm 2} ScaDS.AI Center Leipzig/Dresden\\
    maurice.funk@uni-leipzig.de, daumantas.kojelis@uni-leipzig.de
}

\usepackage{amsmath}
\usepackage{amssymb}
\usepackage{amsthm}
\usepackage{amsfonts}
\usepackage[makeroom]{cancel}
\usepackage[T1]{fontenc} 
\usepackage{thm-restate}

\usepackage{xcolor}

\usepackage{tikz}

\theoremstyle{plain}
\newtheorem{lemma}{Lemma}
\newtheorem{theorem}[lemma]{Theorem}

\newtheorem{proposition}[lemma]{Proposition}

\theoremstyle{definition}
\newtheorem{definition}[lemma]{Definition}
\newtheorem{example}[lemma]{Example}

\theoremstyle{remark}
\newtheorem{remark}[lemma]{Remark}

\newcommand{\Rsim}{\ensuremath{\sim^{L, c, *}_{\exists}}}
\newcommand{\Rsimq}{\ensuremath{\sim^{L, c, q'}_{\exists}}}
\newcommand{\RsimP}[1]{\ensuremath{\sim^{#1, c, *}_{\exists}}}
\newcommand{\RsimPP}[2]{\ensuremath{\sim^{#1, c, #2}_{\exists}}}
\newcommand{\RsimPPP}[3]{\ensuremath{\sim^{#1, #2, #3}_{\exists}}}

\newcommand{\tp}{\ensuremath{\mathrm{tp}_{\sim^{L, c}}}}
\newcommand{\tpP}[1]{\ensuremath{\mathrm{tp}_{\sim_{#1, c}}}}
\newcommand{\tpmo}{\ensuremath{\mathrm{tp}_{\sim^{L-1, c}}}}

\newcommand{\hG}{\ensuremath{\widehat{G}}}

\newcommand{\gnn}{\ensuremath{\mathcal{N}}}

\newcommand{\msl}{\ensuremath{\{\!\{}}
\newcommand{\msr}{\ensuremath{\}\!\}}}

\newcommand{\comb}{\ensuremath{\mathsf{comb}}}
\newcommand{\agg}{\ensuremath{\mathsf{agg}}}
\newcommand{\rea}{\ensuremath{\mathsf{read}}}
\newcommand{\cls}{\ensuremath{\mathsf{cls}}}

\newcommand{\tG}{\ensuremath{\widetilde{G}}}
\newcommand{\tV}{\ensuremath{\widetilde{V}}}
\newcommand{\tE}{\ensuremath{\widetilde{E}}}
\newcommand{\tf}{\ensuremath{\widetilde{f}}}

\newcommand{\Ctwo}{\ensuremath{\text{C}^2}}
\newcommand{\GML}{\ensuremath{\text{GML}}}

\begin{document}

\maketitle

\begin{abstract}
    We study the expressive power of message-passing aggregate-combine-readout graph neural networks (ACR-GNNs).
    Particularly, we focus on the first-order (FO) properties expressible
    by this formalism.
    While a tight logical characterisation remains a difficult
    open question, we make two contributions towards answering it.
    First, we show that sum aggregation and readout suffice for GNNs to
    capture FO properties that cannot be expressed in the logic
    $\Ctwo$ on both directed and undirected graphs. This strengthens known
    results by \citet{Hauke_Wałęga_2026} where aggregation and readout functions are specially crafted for the task.
    Second, we identify two natural ways of restoring characterisability (in regards to $\Ctwo$) for ACR-GNNs.
    One option is to limit local aggregation (without imposing restrictions on global readout), whilst
    the second is to run ACR-GNNs over graphs of bounded degree (but unbounded size).
    In both cases, the FO properties captured by GNNs are exactly
    those definable by a formula in graded modal logic with global counting modalities.
    Our results thus establish an innate lower- and upper-bound in terms of how far (fragments of) $\Ctwo$ can be
    taken to characterise GNNs, and imply that is indeed the unbounded interaction of aggregation and readout that
    pushes the logical expressive power of GNNs above $\Ctwo$.
\end{abstract}

\section{Introduction}

Message-passing graph neural networks
(GNNs)~\cite{scarselliGraphNeuralNetwork2009} are popular machine learning
models for graph structured data and are by design permutation-invariant. They
have found applications in diverse areas that involve graph data such as analysis
of molecules~\cite{DuvenaudMABHAA15}, medicine~\cite{ZitnikAL18}, fraud detection~\cite{DengH21}, and
traffic forecasting~\cite{JiangL22}.

Many different variants of GNNs are used in practice, but they
all share central features. They associate each vertex of a graph with an
embedding vector and update this embedding layer-by-layer.
In each layer, every vertex sends messages containing its current embedding to
other vertices, and receives such messages themselves. These received messages
are then aggregated and combined with the current vertex embedding to form the
new embedding.

The main differences between GNN variants lie in which vertices receive a
message and how these messages are aggregated and combined. 
Commonly, messages are only passed along edges of the graph. Other variants are equipped with
global readout functionality -- the ability to aggregate messages from \textit{all} vertices regardless of their connection in the graph.
Popular choices of functions used for processing the messages in aggregation and
readout are dimension-wise sum, max, or mean.
The combination function, which forms the new vertex embeddings from the
aggregated messages, is typically a trainable feed-forward neural network (FFNN).
To understand the limitations and relationships of all these possible choices, a growing body
of literature is dedicated to characterising the expressive power of the GNN variants using logical formalisms.

One main direction of recent research is finding which vertex (or graph)
properties can be expressed by GNNs. The central result
in this area is by \citet{Barcel2020},
who show that GNNs with only local aggregation can express exactly the
first-order (FO) vertex properties that are formulas of graded modal logic
($\GML$), and that sum aggregation suffices for this.
In the same paper, the authors also show that GNNs with both local aggregation and global
readout can express all formulas of the logic $\Ctwo$, that is, the two-variable fragment
of first-order logic with counting quantifiers. The question whether $\Ctwo$
exactly corresponds to GNNs with aggregation and readout was left open.

Surprisingly, the questions was later answered negatively by
\citet{Hauke_Wałęga_2026}. They show that there is an FO vertex property that
can be expressed by certain aggregate-combine-readout GNN (ACR-GNN)
architectures, but cannot be expressed
by $\Ctwo$.
In light of this unexpected expressive power, it is an open question what
logics can be used to characterise GNNs with readout.


More is known when the aggregation and readout functions of ACR-GNNs are not
arbitrary functions, but are bounded in the sense that they are only sensitive
to messages up to a certain multiplicity. \citet{GrauFW26} show that ACR-GNNs with
bounded readout and bounded aggregation correspond exactly to $\GML$ extended with
a global counting modality. This holds for all vertex properties, not just
the ones expressible in FO.

Unfortunately, common aggregation and readout functions such as sum and mean are
not bounded. Likewise, it is not known whether these specific aggregation functions allow
ACR-GNNs to express FO properties outside of $\Ctwo$.
This leaves a gap in the understanding of the expressive power of GNNs,
as the known upper and lower bounds are relatively far apart.
We make two contribution towards closing this gap.

First, we show that \textit{simple} ACR-GNNs which only use (unbounded) sum for
aggregation and readout can express FO properties outside of $\Ctwo$.
We show this for both directed (Theorem~\ref{theorem_C2_over_directed_inexpress}) and undirected graphs
(Theorem~\ref{theorem_C2_over_undirected_inexpress}). This strengthens the
mentioned result by~\citet{Hauke_Wałęga_2026}, which relies on specialized
aggregation, readout and combination functions.
In fact, for directed graphs we use the same FO property as
\citeauthor{Hauke_Wałęga_2026} (strict linear orders), but use a new
characterisation in terms of \emph{homomorphism counts} to show that this
property is expressible using only sum aggregation and readout.
For undirected graphs, we also follow a similar approach as
\citeauthor{Hauke_Wałęga_2026} and simulate directed
edges via undirected gadgets. Compared to the paper above, we use a different
directed graph gadgetisation for technical reasons.

As our second contribution, we identify two relevant cases that each 
result in ACR-GNNs capturing exactly the FO vertex properties expressible in $\GML$
extended a global counting modality on both directed and undirected graphs.
The first case is the restriction to only graphs in which the degree of vertices
is bounded by some constant (Theorem~\ref{prop bounded graphs gml}). This is a
natural assumption for many graph datasets, for example, those that originate from
molecules or road networks.
The second, closely related, case is the case of ACR-GNNs in which the
aggregation functions are bounded but the readout functions can be arbitrary
(Proposition~\ref{theo_main_abstract_logic}).
This, for FO vertex properties, generalizes the result by \citet{GrauFW26},
where \textit{both} aggregation and readout functions are taken to be bounded.
%
It is quite surprising that allowing unbounded readout functions
does not increase the expressive power relative to FO. This reveals that
it is indeed the unbounded interaction of readout and aggregation that
moves GNNs outside of $\Ctwo$ in terms of expressivity.
Our characterisation result relies on the combination of first-order model theory
and architectural properties of GNNs. The constructions used may be of independent
interest in the study of finite model theory.

The rest of this paper is structured as follows. 
Section~\ref{sec:prelim} gives the necessary definitions used throughout the work.
In Section~\ref{section_directed_graphs}, we show that ACR-GNNs that use only
sum for aggregation, readout and combination can express a first-order
property that is outside of $\Ctwo$.
In Section~\ref{sec:undirected} we transfer the result to the setting of undirected graphs.
Section~\ref{section_abstract_logic} then is concerned with the characterisation
of ACR-GNNs using $\GML$ with global counting modalities in the case of graphs
of bounded degree and the case of bounded aggregation functions.
Section~\ref{sec:conclusion} concludes the paper.




\paragraph{Related Work} 
Our results can be viewed as the continuation of the following three works.
\citet{Barcel2020} show that GNNs (without readout) are equivalent in expressive power to graded modal logic (relative to FO).
\citet{Hauke_Wałęga_2026} show that no such equivalence can be established for ACR-GNNs and any fragment of $\Ctwo$.
\citet{GrauFW26} complements the above by providing logical characterisations for ACR-GNNs with bounded aggregation and readout.

In terms of exact characterisations, we mention the work by
\citet{BenediktLMT24}, where expressive power of GNNs with sum
aggregation and bounded activation functions 
is shown to be that of modal logic with Presburger quantifiers.
A result by \citet{tenacucalaCorrespondenceMonotonicMaxsum2023} translates GNNs with max aggregation
into a form of datalog that is close to modal logic.
\citet{walegaPreservationTheoremsUnravellingInvariant2026} identify subclasses
of GNNs that correspond exactly to fragments of $\GML$.
\citet{SchonherrL26} show that GNNs with mean aggregation are related to ratio modal logic, albeit in the non-uniform setting.
When relativised to FO, such GNNs match standard modal logic.

With regard to global readout,
\citet{rosenbluth2024} compare the functions that can be computed by ACR-GNNs with
sum readout and self-attention (which can be considered a form of weighted mean
readout).
\citet{ahvonen2026} show that graph transformers and
GNNs with mean readout, when relativised to FO, match graded modal logic with a
global existence modality. 
An upper bound on the expressive power of ACR-GNNs with sum readout is provided
by \citet{grohe2024} using the logic FO+C which extends FO with very expressive
counting terms.

Another line of work is interested in the distinguishing power of GNNs. Here the
foundational result is that this distinguishing power is bounded by the
1-dimensional Weisfeiler Lemann test~\cite{xu2018howPowerful}.
There are also results on the expressive power of recurrent GNNs that do not
have a fixed number of message passing rounds~\cite{PfluegerCK24,ahvonen2024}. 
\citet{soeteman2026} investigate the expressive power of GNNs for graphs that
are equipped with unique node identifiers.

\section{Preliminaries}\label{sec:prelim}

A multiset $M$ over some set $S$ is a function $S \to \mathbb{N}$. We use $\msl
\cdot \msr$ to denote multisets and $\mathcal{M}(S)$ to denote the set of all
multisets over $S$.
For a multiset $M$ over $S$, $M_{|c}$ denotes the $c$-restriction of $M$, that is, the
multiset given by $M_{|c}(x) = \min(M(x), c)$ for all $x \in S$.
We say that a function $f$ with domain $\mathcal{M}(S)$ is \emph{$c$-bounded} for some $c \in \mathbb{N}$
if $f(M) = f(M_{|c})$ for all $M \in \mathcal{M}(S)$.
We sometimes leave $c$ implicit and simply say that $f$ is \textit{bounded}.

A common bounded function on multisets of vectors is the operation taking the
dimension-wise maximal value $\max$. In fact, $\max$ is $1$-bounded. On the
other hand, functions that take the dimension-wise sum, product or mean of a
given multiset are not $c$-bounded for any $c \in \mathbb{N}$.

If $v$ is an $d$-dimensional vector, we denote its $i$-th component by $v[i]$.

\paragraph{Graphs}
A graph $G = (V, E)$ is a tuple consisting of a non-empty set of vertices $V$ and edges $E \subseteq V \times V$.
Unless explicitly stated, graphs are \emph{directed} and can contain reflexive edges.
In \emph{undirected} graphs $(v, u) \in E$ implicitely means that also $(u, v) \in E$.
The \emph{out-neighbourhood} of a vertex $v \in V$, denoted $N_{\text{out}}^G(v)$, is the set of all vertices reachable from $v$ with an outgoing edge.
More formally, $N_{\text{out}}^G(v) = \{ u \in V \mid (v, u) \in E \}$. The \emph{in-neighbourhood} $N_{\text{in}}^G(v)$ is defined
similarly. When dealing with undirected graphs, the in- and out-neighbourhoods
coincide and are simply referenced as $N^G(v)$.
A graph has \emph{bounded degree} $c$ if $|N_{\text{out}}^G(v)| \leq c$ for all $v \in V$.

We say that a graph $G$ is \emph{$d$-featured} (for some $d \in \mathbb{N}$) if it is
associated with a one-hot embedding $f \colon V \to \{0, 1\}^d$. We then write $G = (V, E, f)$.
A \emph{pointed} graph is a tuple $(G, v)$ where $G = (V, E)$ and $v \in V$. We
also allow pointed graphs to be $d$-featured.

\paragraph{Graph Neural Networks}
The GNNs concerned in this paper are of the aggregate-combine-readout variety.
Formally, an $L$-layer \emph{ACR-GNN} with input dimension $d \in \mathbb{N}$ is a
$(3L{+}1)$-tuple $\gnn = (\{\comb_i\}_{i \leq L}, \{\agg_i\}_{i \leq L}, \{\rea_i\}_{i
\leq L}, \cls)$, such that there exist $d_0, \ldots, d_L \in \mathbb{N}$ with $d = d_0$ and
for all $i \leq L$, $\agg_i$ and $\rea_i$ are functions
$\mathcal{M}(\mathbb{R}^{d_{i - 1}}) \to \mathbb{R}^{{d_{i - 1}}}$, and
$\comb_i$ is a function $\mathbb{R}^{d_{i - 1}} \times \mathbb{R}^{d_{i - 1}} \times
\mathbb{R}^{d_{i - 1}} \to \mathbb{R}^{d_i}$.
We refer to the $\comb_i$ as combination functions, the $\agg_i$ as aggregation
functions, and the $\rea_i$ and readout functions.
The last element $\cls$ is a binary classifier $\mathbb{R}^{d_L} \to \{0, 1\}$.
With an \emph{aggregate-combine GNN} (AC-GNN) we mean an ACR-GNN where all
readout functions are constant.

On a $d$-featured graph $G = (V, E, f)$ an $L$-layer ACR-GNN with input dimension
$d$ computes a series of functions $f^{(i)} \colon V \to \mathbb{R}^{d_i}$ for $i \leq L$, starting with
$f^{(0)} = f$, and then computing 
\begin{align*}
    f^{(i)}(v) := \comb_i \Big( &f^{(i - 1)}(v), \\
     &\agg_i \big( \msl f^{(i - 1)}(u) \mid u \in N_\text{out}^G(v) \msr \big), \\ 
     &\rea_i \big(\msl f^{(i - 1)}(u) \mid u \in V \msr \big) \Big),
\end{align*}
for all $v \in V$ and $i \leq L$.

An $L$-layer ACR-GNN $\gnn$ with input dimension $d$ \emph{accepts} a pointed $d$ featured
graph $(G, v)$ if $\cls(f^{(L)}(v)) = 1$.
To avoid notational clutter we will write $\gnn(G, v) = 1$ when $\gnn$ accepts $(G, v)$, and $\gnn(G, v) = 0$ when it does not.

\paragraph{Simple ACR-GNNs}
Recall that an $L$-layer ACR-GNN $\gnn$ is the tuple
$(\{\comb_i\}_{i \leq L}, \{\agg_i\}_{i \leq L}, \{\rea_i\}_{i \leq L}, \cls)$.
Following \citet{Barcel2020}, we say that $\gnn$ is \textit{simple}
if every $\mathsf{comb}_i$ is a function of the form
\[
    \mathsf{comb}_i(x_1, x_2, x_3) = \mathrm{ReLu}(x_1 A_i + x_2 C_i + x_3 R_i + b_i)
\]
where $A_i, C_i, R_i$ are matrices and $b_i$ is a vector of the appropriate dimensions, 
$\mathrm{ReLu}(x) = \max(0, x)$ is applied dimension-wise; every function
$\mathsf{agg}_i$ and $\mathsf{read}_i$ is sum; and $\mathsf{cls}\colon \mathbb{R} \to
\{0, 1\}$ is a linear threshold classification function.

\paragraph{Graded Modal Logic}
Let us fix a $d \in \mathbb{N}$.
We will make heavy use of $d$-dimensional graded modal logic ($d$-dimensional $\GML$)
as well as $d$-dimensional graded modal logic with graded global modalities ($d$-dimensional $\GML^\exists$).
Formally, $d$-dimensional $\GML$ is the set of all formulas $\varphi$ over the following grammar:
\begin{equation*}
    \varphi ::= \top \mid p_i \mid \neg \varphi \mid \varphi \wedge \varphi \mid \Diamond_{\geq k} \varphi,
\end{equation*}
it being understood that $k \in \mathbb{N}$ and $p_i$ is a proposition letter indexed by integers $i \in [1, d]$.
In $d$-dimensional $\GML^\exists$ we additionally allow $\varphi$ to take the form $\exists_{\geq k} \varphi$.
For readability, we use $\varphi \vee \psi$ in place of $\neg (\neg \varphi \wedge \neg \psi)$.

We evaluate formulas of $d$-dimensional $\GML$ and $\GML^\exists$ over pointed $d$-featured graphs $(G, v)$.
Let $G = (V, E, f)$. We define satisfaction of a formula $\varphi$ at vertex $v$ in $G$,
written $G, v \models \varphi$, inductively as follows:
\begin{align*}
    &G, v \models \top  && \text{always} \\
    &G, v \models p_i && \text{ if } f(v)[i] = 1\\
    &G, v \models \neg \varphi && \text{ if } G, v \not\models \varphi \\
    &G, v \models \varphi_1 \wedge \varphi_2 && \text{ if } G, v \models \varphi_i \text{ for } i \in \{1, 2\}\\
    &G, v \models \Diamond_{\geq k} \varphi &&  \text{ if } |\{ u \in N_{\text{out}}(v) \mid G, u \models \varphi\}| \geq k \\
    &G, v \models \exists_{\geq k} \varphi &&  \text{ if } |\{ u \in V \mid G, u \models \varphi\}| \geq k.
\end{align*}
In the sequel we will suppress reference to $d$-dimensionality if it can be inferred from context.


\paragraph{First-order Logic}

We will also be concerned with vertex properties expressible in \textit{first-order logic} (FO).
Briefly, the language of FO over graphs allows to freely quantify over vertices and describe
desired features, edges, and identities. Of particular interest is the \textit{two-variable fragment (of FO) with counting}
($\Ctwo$). Here, only FO formulas that can be written in two variables
(and counting quantifiers reminiscent of those in $\GML^\exists$) are permitted.
It is known that every $\GML^\exists$ formula is equivalent to a formula in $\Ctwo$.
The converse, however, is not true.
Given a featured pointed graph $(G, v)$ and a formula $\varphi$ with one free variable in any of the above languages,
we continue to write $G, v \models \varphi$ to signal that $\varphi$ is satisfied by $v$ in $G$.
We say that $\varphi$ is of \textit{quantifier rank} $q$ if the maximal nesting depth of quantifiers in $\varphi$
is at most $q$.

In the sequel we will be concerned with the capabilities of formal languages in terms of
classifying featured pointed graphs. Thus, we call FO, $\Ctwo$, and $\GML^\exists$ formulas over
the language of $d$-featured pointed graphs FO, $\Ctwo$, and, respectively, $\GML^\exists$ \textit{classifiers}.
We say that a classifier $\varphi$ of the above sort (or, respectively, an ACR-GNN $\gnn$) \emph{captures} a property $\mathcal{P}$ if for all
featured pointed graphs $(G, v)$, we have $G, v \models \varphi$ (respectively, $\gnn(G, v) = 1$) if and only if $(G, v)$
has the property $\mathcal{P}$. Note that the property $\mathcal{P}$ can itself be defined by a classifier.
\section{ACR-GNNs on Directed Graphs}\label{section_directed_graphs}

We begin by investigating the expressive power of ACR-GNNs on directed graphs.
The main result of this section is the following.
\begin{theorem}\label{theorem_C2_over_directed_inexpress}
    There is an FO vertex property over directed graphs that is not equivalent
    to any $\Ctwo$ formula but is captured by a simple ACR-GNN.
\end{theorem}

To show this, we briefly recall the approach used by \citet{Hauke_Wałęga_2026}
for their result.
They consider the FO property $\varphi(x)$ that is satisfied by $(G, v)$ if and
only if the edge relation of $G$ is a strict linear order. For example, $\varphi$ can be the
conjunction of
\begin{align*}
   & \forall x \ \neg E(x, x)  && \text{(irreflexive)} \\
   & \forall x \forall y \ \big( x = y \lor E(x, y) \lor E(y, x) \big)&& \text{(total)} \\
   & \forall x \forall y \forall z \big( E(x, y) \land E(y, z) \to E(x, z) \big)  && \text{(transitive)}.
\end{align*}

Observe that the subformula for transitivity uses three variables and therefore
does not lie in the two-variable fragment of FO or in $\Ctwo$. Furthermore one
can show the following.

\begin{proposition}[\citet{Hauke_Wałęga_2026}]\label{prop: c2 cannot express linear order}
    There is no $\Ctwo$ formula that is equivalent to $\varphi$.
\end{proposition}

\citeauthor{Hauke_Wałęga_2026} continue to show that there is an ACR-GNN that
captures $\varphi$. For this, they characterise strict linear orders in way that
can be checked by an ACR-GNN. This constructed ACR-GNN uses relatively complex
combination, aggregation and readout functions that involve operations such as
exponentiation and which cannot be exactly expressed by FFNNs or by the standard
aggregation functions mean, max and sum.
It thus remained an open question whether ACR-GNNs that use only standard
functions can also capture $\varphi$.

As a first step to answer this question, one can use the characterisation of
linear orders used by \citeauthor{Hauke_Wałęga_2026} to show that only max and
sum as aggregation and readout functions suffice to express $\varphi$. Here, we
go one step further and show that sum suffices as the only aggregation and
readout function. This requires a different characterisation of strict linear
orders, which we introduce next.




A \emph{homomorphism} from a graph $G_1 = (V_1, E_1)$ to a graph $G_2 = (V_2,
E_2)$ is a function $h \colon V_1 \to V_2$ such that $(u, v) \in E_1$ implies
$(h(u), h(v)) \in E_2$.
With $\hom(G_1, G_2)$ we mean the set of all homomorphisms from $G_1$ to $G_2$.
Furthermore, we use $P_2$ to denote the path-shaped graph of length $2$, that is,
$P_2 = (\{1, 2, 3\}, \{ (1, 2), (2, 3)\})$.

\begin{restatable}{lemma}{lemtotalorderproperties}\label{lem:total-order-properties}
    A graph $G = (V, E)$ with $|V| = n$ is a strict linear order if and only if
    $|E| = \binom{n}{2}$ and $|\hom(P_2, G)| = \binom{n}{3}$.
\end{restatable}

It is relatively easy to verify that strict linear orders satisfy the two
properties. The other direction is more difficult to show, and, was as far as we
are aware previously only known for the case where $G$ is a
tournament~\cite{hararyTheoryRobinTournaments1966}.
We provide a detailed proof in the appendix.

The intuition behind this characterisation is that every deviation from a strict
linear order while maintaining $|E| = \binom{n}{2}$, increases the number of
homomorphic matches of $P_2$. As an example, consider the following three graphs $G_1, G_2,
G_3$ with three vertices each.
    \begin{center}
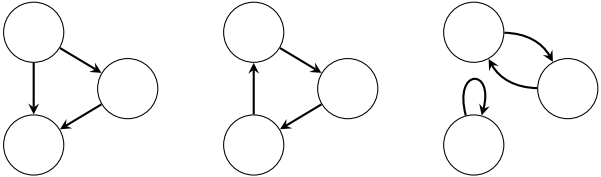
        \begin{tabular}{ccc}
        \begin{tikzpicture}[
            baseline = (1),
            node/.style={circle, draw, minimum size=8mm},
            edge/.style={->, thick, >=stealth},
        ]
            \node[node] (1) at (0, 0) {};
            \node[node] (2) at (1.25, -0.75) {};
            \node[node] (3) at (0, -1.5) {};
            \draw[edge] (1) to (2);
            \draw[edge] (1) to (3);
            \draw[edge] (2) to (3);
        \end{tikzpicture} \quad &  \quad
        \begin{tikzpicture}[
            baseline = (1),
            node/.style={circle, draw, minimum size=8mm},
            edge/.style={->, thick, >=stealth},
        ]
            \node[node] (1) at (0, 0) {};
            \node[node] (2) at (1.25, -0.75) {};
            \node[node] (3) at (0, -1.5) {};
            \draw[edge] (1) to (2);
            \draw[edge] (3) to (1);
            \draw[edge] (2) to (3);
        \end{tikzpicture} \quad &  \quad
        \begin{tikzpicture}[
            baseline = (1),
            node/.style={circle, draw, minimum size=8mm},
            edge/.style={->, thick, >=stealth},
        ]
            \node[node] (1) at (0, 0) {};
            \node[node] (2) at (1.25, -0.75) {};
            \node[node] (3) at (0, -1.5) {};
            \draw[edge, bend left] (1) to (2);
            \draw[edge, bend left] (2) to (1);
            \draw[edge, loop above] (3) to (3);
        \end{tikzpicture}
    \end{tabular}
    \end{center}
The graph $G_1$ is a strict total order and has the minimal number $\binom{3}{3}
= 1$ of homomorphic matches of $P_2$. Inverting an edge, as in $G_2$, increases
the number of matches to $3$. The same happens when reflexive loops or cycles of
length $2$ are introduced: $G_3$ has $3$ matches of $P_2$.


It is interesting to observe that Lemma~\ref{lem:total-order-properties}
characterises linear orders using \emph{homomorphism counts}. There exists an
interesting range of results on what properties can be expressed using
homomorphism counts, although these mostly involve an infinite number of patterns
(e.g.~\citet{dvorakRecognizingGraphsNumbers2010}).
Most relevant is perhaps Theorem 4.4 of~\citet{chenAlgorithmsBasedFinitely2025},
which implies that
strict linear orders can be characterized using a finite number of homomorphisms
counts. However, by itself, it is not clear whether this characterization
involves patterns that can be counted by ACR-GNNs, and therefore cannot directly
replace Lemma~\ref{lem:total-order-properties} for our purposes.

Next, we show that the characterisation of strict linear orders in
Lemma~\ref{lem:total-order-properties} allows simple ACR-GNNs to capture
$\varphi$.


\begin{proposition}\label{lemma_simple_gnn_directed_orders}
    There is a simple 6-layer ACR-GNN that accepts $(G, v)$ if and only if
    $G$ is a strict linear order.
\end{proposition}

\begin{proof}
    The desired simple ACR-GNN verifies that the two properties from
    Lemma~\ref{lem:total-order-properties} hold in $G$. 
    Writing $n = |V|$ and taking $G$ to be $0$-featured,
    we first show that the
    required individual values can be computed by simple $4$-layer ACR-GNNs $\gnn_1, \gnn_2, \gnn_3, \gnn_4$.
    Specifically, we want $\gnn_i$ to compute the 1-dimensional feature embedding $f^{(4)}_i \colon V \to \mathbb{R}$ such that, for all $v \in V$:
    \begin{enumerate}
        \item $f^{(4)}_1(v) = |E|$,
        \item $f^{(4)}_2(v) = \binom{n}{2}$,
        \item $f^{(4)}_3(v) = |\hom(P_2, G)|$,
        \item $f^{(4)}_4(v) = \binom{n}{3}$.
    \end{enumerate}

    One can then compose the simple ACR-GNNs $\gnn_1$, $\gnn_2$, $\gnn_3$, and $\gnn_4$ into a single
    4-layer simple ACR-GNN that performs the same computations in separate dimensions.
    Two additional layers then perform the required equality checks using the
    combination function. More precisely, the following
    embedding vector $\bar x$ is computed at layer 5 for all $v \in V$:
    \begin{align*}
        \mathrm{ReLu}\big((|E|& - \textstyle\binom{n}{2}, \textstyle\binom{n}{2} - |E|,
        \\&|\hom(P_2, G)| - \textstyle\binom{n}{3}, \textstyle\binom{n}{3} - |\hom(P_2, G)|)\big),
    \end{align*}
    which can be expressed in terms of $f^{(4)}_1, \dots, f^{(4)}_4$.
    Clearly, $x[1] + x[2] = 0$ when $|E| = \binom{n}{2}$ and $x[3] + x[4] = 0$ when $|\hom(P_2, G)| = \binom{n}{3}$.
    Since $\mathrm{ReLu}$ guarantees that each $x[i]$ is non-negative, we have
    that by Lemma~\ref{lem:total-order-properties}, $x[1] + x[2] + x[3] + x[4] \leq 0$ if and only if the graph is a strict linear
    order.
    Thus, by computing the above sum using the combination function of layer 6
    we can safely classify the graph as a linear order if
    the final embedding is non-positive.

    Finally, we build the advertised simple ACR-GNNs $\gnn_1, \gnn_2, \gnn_3, \gnn_4$: 
    
    The simple ACR-GNN $\gnn_1$ first computes $f^{(1)}_1(v) = 1$ for all $v \in V$
    by using the bias vector of the combination function and ignoring
    aggregation and readout.
    The second layer computes $f^{(2)}_1(v) = \sum_{u \in N^G_{\mathrm{out}}(v)} f^{(1)}_1(u)$
    which is $|N^G_{\mathrm{out}}(v)|$, using the aggregation function.
    The third layer uses readout (without aggregation) to compute
    $f^{(3)}_1(v) = \sum_{u \in V} f^{(2)}_1(u)$ which is $|E|$.
    The last layer is redundant, $f^{(4)}_1(v) = f^{(3)}_1(v)$.

    The simple ACR-GNN $\gnn_2$ does the following for all $v \in V$.
    Again let $f_2^{(1)}(v) = 1$  by using the bias vector.
    The second layer uses readout to compute $f_2^{(2)}(v) = \sum_{u \in V} f_2^{(1)}(v)$ which is $n$.
    The third layer then uses readout to compute
    $f^{(3)}_2(v) = \frac{-1}{2} f^{(2)}_2(v) + \frac{1}{2} \sum_{u \in V} f^{(2)}_2(u)$
    which is $\frac{-1}{2} n +  \frac{1}{2} n^2 = \frac{n^2 - n}{2} = \frac{n (n - 1)}{2} = \binom{n}{2}$.
    Lastly, $f^{(4)}_2(v) = f^{(3)}_2(v)$.

    For each $v \in V$ the simple ACR-GNN $\gnn_3$ computes the following.
    First, again, $f^{(1)}_3(v) = 1$.
    The second layer uses aggregation to compute
    $f^{(2)}_3(v) = \sum_{u \in N^G_{\mathrm{out}}(v)} f^{(1)}_3(u)$ which is $|N^G_{\mathrm{out}}(v)|$.
    The third layer uses aggregation to compute
    $f^{(3)}_3(v) = \sum_{u \in N^G_{\mathrm{out}}(v)} f^{(2)}_3(v)$
    which is $\sum_{u \in N^G_{\mathrm{out}}(v)} |N^G_{\mathrm{out}}(u)|$ and equals
    the number of length 2 paths originating from $v$.
    The fourth layer uses readout to compute
    $f^{(4)}_3(v) = \sum_{u \in V} f^{(3)}_3(u)$ which equals $|\hom(P_2, G)|$.

    For each $v \in V$ the simple ACR-GNN $\gnn_4$ computes the following. Again,
    $f^{(1)}_4(v) = 1$. Via readout,
    $f^{(2)}_4(v) = n$.
    Via readout and the combination function, the following pair is computed: $f^{(3)}_4 = (\sum_{u \in V} f^{(2)}_4(u), f^{(2)}_4(v))$
    which equals $(n^2, n)$.
    Then, using readout once more,
    $f^{(4)}_4(v) = \frac{1}{6} \sum_{u \in V} f^{(3)}_4(u)[1] + \frac{-3}{6} f^{(3)}_4(v)[1] + \frac{2}{6} f^{(3)}_4(v)[2]$.
    This is then $ \frac{n^3 - 3n^2 + 2n}{6} = \frac{n (n - 1) (n - 2)}{6} = \binom{n}{3}$, as required. \qedhere
\end{proof}

Proposition~\ref{prop: c2 cannot express linear order} and
Proposition~\ref{lemma_simple_gnn_directed_orders} together now imply
Theorem~\ref{theorem_C2_over_directed_inexpress}. 

\section{ACR-GNNs on Undirected Graphs}
\label{sec:undirected}

In the same style as~\citet{Hauke_Wałęga_2026} we now transfer
Theorem~\ref{theorem_C2_over_directed_inexpress} to undirected graphs.
More precisely, we obtain the following result.
\begin{theorem}\label{theorem_C2_over_undirected_inexpress}
    There is an FO vertex property over undirected graphs that
    is not equivalent to any $\Ctwo$ formula, but is captured by a simple ACR-GNN.
\end{theorem}
Complementing Theorem~\ref{theorem_C2_over_undirected_inexpress}, it is known
that every $\Ctwo$ formula is captured by a simple ACR-GNN over undirected
featured graphs~\cite{Barcel2020}.

To show Theorem~\ref{theorem_C2_over_undirected_inexpress}, we encode directed
graphs as undirected graphs. For this, we use the following construction.
\begin{definition}\label{def:gadget}
    The \emph{gadgetisation} of a directed graph $G = (V, E)$
    is a 2-featured undirected graph $\tG = (\tV, \tE, \tf)$
    such that $\tV = \tV_{s} \cup \tV_{t} \cup \tV_{id}$, where
    \begin{itemize}
        \item $\tV_{s} = \{ s_v \mid v \in V \}$,
        \item $\tV_{t} = \{ t_v \mid v \in V \}$, and
        \item $\tV_{id} = \{ \iota_v \mid v \in V \}$
    \end{itemize}
    are pairwise disjoint sets, $\tE = \tE_{st} \cup \tE_{id}$, where
    \begin{itemize}
        \item $\tE_{st} = \{ (s_v, t_u) \mid (v, u) \in E  \}$,
        \item $\tE_{id} = \{ (s_v, \iota_v), (\iota_v, t_v) \mid v \in V \}$; and
    \end{itemize} 
    $\tf(s_v) = (1, 0)$, $\tf(t_v) = (0, 1)$, $\tf(\iota_v) = (0, 0)$ for $v \in V$.
\end{definition}
    
\begin{figure}
    \centering
    \hspace*{5mm} \begin{tikzpicture}[
    v_node/.style={circle, draw, minimum size=8mm},
    iota_node/.style={circle, draw, fill=green!20, minimum size=8mm},
    t_node/.style={circle, draw, fill=orange!20, minimum size=8mm},
    edge/.style={->, thick, >=stealth}
]
    \node[v_node] (v1) at (0,2.5) {$1$};
    \node[v_node] (v2) at (0,1.25) {$2$};
    \node[v_node] (v3) at (0,0) {$3$};

    \draw[edge] (v1) -- (v2);
    \draw[edge] (v2) -- (v3);
    \draw[edge] (v1) to[bend left=45] (v3);
\end{tikzpicture} \hspace*{7mm}
    \begin{tikzpicture}[
    s_node/.style={circle, draw, fill=blue!20, minimum size=8mm},
    iota_node/.style={circle, draw, fill=green!20, minimum size=8mm},
    t_node/.style={circle, draw, fill=orange!20, minimum size=8mm},
    edge/.style={-, thick, >=stealth}
]
    \foreach \i [evaluate=\i as \y using {2-\i*1.25}] in {1,2,3} {
        \node[s_node] (s\i) at (0,\y) {$s_{\i}$};
        \node[iota_node] (iota\i) at (1.5,\y) {$\iota_{\i}$};
        \node[t_node] (t\i) at (3,\y) {$t_{\i}$};
    }
    \draw[edge] (s1) -- (iota1);
    \draw[edge] (iota1) -- (t1);

    \draw[edge] (s2) -- (iota2);
    \draw[edge] (iota2) -- (t2);

    \draw[edge] (s3) -- (iota3);
    \draw[edge] (iota3) -- (t3);

    \draw[edge] (s1) -- (t2);

    \draw[edge] (s2) -- (t3);

    \node (A) at (3.7,1) {};
    \draw[edge] (s1) to[bend left=45] (A) to[bend left=51] (t3);
\end{tikzpicture}
    \caption{ A strict linear order $G$ (left) over vertices $\{1, 2, 3\}$ and its gadgetisation $\tG$ (right).
Vertex colourings indicate the features attributed by $\tf$.}
\label{fig: gadget example}
\end{figure}
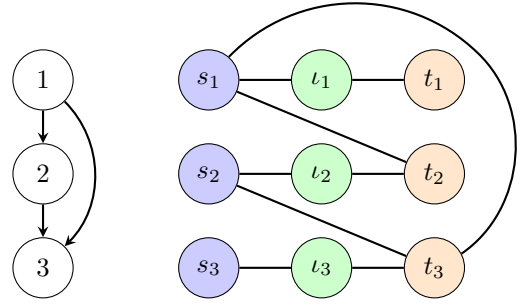

Intuitively, the gadgetisation $\tG$ of a directed graph $G = (V, E)$,
is, when restricted to $\tV_{s} \cup \tV_{t}$, a bipartite graph
in which each vertex $v \in V$ is split into `source' $s_v$ and `sink' $t_v$ components.
We think of edges in $\tE$ involving $t_v$ as incoming, and edges having $s_v$ as outgoing from $v$.
The vertices $\tV_{id}$ help identifying which pairs of $\tV_{s} \times \tV_{t}$
correspond to a single vertex in $V$. The feature embedding $\tf$ keeps track
of 
of which vertex is in which of the sets $\tV_{s}, \tV_{t}, \tV_{id}$.
See the example in Figure~\ref{fig: gadget example}.

\begin{remark}
    The gadgetisation used in \cite[Definition~9]{Hauke_Wałęga_2026} does
    not work with our characterisation of strict linear orders
    in Lemma~\ref{lem:total-order-properties}.
    Using the gadgetisation of the paper mentioned would allow
    multiple (gadgetised) edges starting at some vertex $v$ and ending at some vertex $u$.
    This would then tamper with homomorphism counts.
\end{remark}

The vertex property that we use to show
Theorem~\ref{theorem_C2_over_undirected_inexpress} is then being the
gadgetisation of a strict linear order. By adapting the FO formula $\varphi$ and
the ACR-GNN constructed in the proof of Proposition~\ref{lemma_simple_gnn_directed_orders}
from directed graphs to gadgetisations of directed graphs, we show the following.

\begin{restatable}{proposition}{propgadgetisationsexpressible}\label{prop_gadgetisations_expressible}
    There is an FO classifier and a simple ACR-GNN that capture the
    gadgetisations of strict linear orders.
\end{restatable}
For this, first note that that the property of being a gadgetisation of a
directed graph can be expressed by an FO formula and a simple ACR-GNN. The
simple ACR-GNN we again construct to count the number of homomorphisms. This time,
however, we use the homomorphisms of the gadgetisation of $P_2$.
A full proof of Proposition~\ref{prop_gadgetisations_expressible} is in the appendix.

Finally, we show, by utilising the appropriate bisimulations~\cite{Hella96},
that $\Ctwo$ is incapable of expressing being a gadgetisation of strict linear orders,
thus concluding the proof of Theorem~\ref{theorem_C2_over_undirected_inexpress}.
\begin{restatable}{proposition}{propositioncbisimulation}\label{proposition_C2_bisimulation}
    There is no $\Ctwo$ classifier that captures gadgetisations of
    strict linear orders.
\end{restatable}
The full proofs and the definition of $\Ctwo$-bisimulations can be found in the
appendix.

\section{ACR-GNNs on Graphs of Bounded Degree}\label{section_abstract_logic}

The graphs used in the proofs of
Theorem~\ref{theorem_C2_over_directed_inexpress} and
Theorem~\ref{theorem_C2_over_undirected_inexpress} have a particular structure
that involves vertices of high out-degree. Here, we study the expressive power
of ACR-GNNs if we exclude such graphs, meaning that we bound the degree of the
considered graphs by some constant, but leave the number of vertices unbounded. 

Our main result is the following complete characterization in the spirit of
\citet{Barcel2020}.
\begin{theorem}\label{prop bounded graphs gml}
    Over graphs of bounded degree, an FO classifier $\varphi$ is captured by an ACR-GNN 
    if and only if $\varphi$ is equivalent 
    to a $\GML^{\exists}$ classifier.
\end{theorem}
This contrasts Theorem~\ref{theorem_C2_over_directed_inexpress} and
Theorem~\ref{theorem_C2_over_undirected_inexpress},
as over graphs of unbounded degree, ACR-GNNs can express properties outisde of
$\GML^\exists$ and $\Ctwo$.

To show Theorem~\ref{prop bounded graphs gml}, we take a detour and first
consider again graphs of unbounded degree, but ACR-GNNs with bounded aggregation
functions and unbounded readout. For this setting we obtain a similar result.
\begin{proposition} \label{theo_main_abstract_logic}
    An FO classifier $\varphi$ is captured by an ACR-GNN with bounded aggregation functions if and only if
    $\varphi$ is equivalent to a $\GML^\exists$ classifier.
\end{proposition}
At the end of this section, we then show how
Proposition~\ref{theo_main_abstract_logic} can be used to show Theorem~\ref{prop
bounded graphs gml}.
In a sense, Proposition~\ref{theo_main_abstract_logic} gives an upper-bound on how
far (fragments of) $\Ctwo$ can be taken in terms of capturing the expressive
power of ACR-GNNs (relative to FO).


One direction of the proof of Theorem~\ref{theo_main_abstract_logic} is a
straightforward adaptation of \cite[Proposition~4.1]{Barcel2020}.
\begin{proposition}\label{proposition_easy_part_of_characterisation}
    Every $\GML^\exists$ classifier is captured by an ACR-GNN with bounded aggregation functions.
    
\end{proposition}
The idea is the same as in \cite[Proposition~4.1]{Barcel2020}: subformulas of a $\GML^{\exists}$ classifier $\varphi$
can be enumerated (say, as $\psi_1, \dots, \psi_n$, with $\psi_n = \varphi$) and given a unique position (say, $\psi_i$ is given $i$) in feature vectors.
The goal is to have the $i$-th position of the $i$-th feature embedding of $v \in V$ determine the satisfaction of $\psi_i$
by $G, v$.
This is achieved by carefully selecting weights and biases depending on the kind
of formula in question. It is a direct consequence of the proof of
\cite[Proposition~4.1]{Barcel2020}, that $k$-bounded sum suffices as aggregation
function, where $k$ is the highest number such that $\Diamond_{\geq k}$ occurs
in $\varphi$.

The converse direction is far less obvious and might be of independent interest in the study of finite model theory.
\begin{restatable}{proposition}{propositionhardpartofcharacterisation}\label{proposition_hard_part_of_characterisation}
    Suppose an ACR-GNN with bounded aggregation functions captures an FO classifier $\varphi$.
    Then, $\varphi$ is equivalent to a $\GML^\exists$ formula.
\end{restatable}
To show the above, we first define a suitable notion of bisimulation under which
ACR-GNNs with bounded activation functions are invariant, and then obtain a van
Benthem-Rosen style result that shows that any FO formula that is invariant
under this notion of bisimulation must be equivalent to a 
$\GML^\exists$ formula.
Recall the notion of $c$-graded, $L$-turn bisimulations that is associated with 
$\GML$~\cite{otto2023gradedmodallogiccounting}.
\begin{definition}
For two $d$-featured pointed graphs $G_1 = (V_1, E_1, f_1)$ and $G_2 = (V_2,
E_2, f_2)$, we say that $(G_1, v_1)$ and $(G_2, v_2)$ are \emph{$c$-graded
$L$-turn bisimilar}, written $G_1, v_1 \sim^{L, c} G_2, v_2$, when
    \begin{enumerate}
        \item $f_1(v_1) = f_2(v_2)$;
        \item (forth) for all distinct
        $u_1^{(1)}, \dots, u_1^{(k)} \in V_1$ such that $k \leq c$ and $(v_1,u_1^{(i)}) \in E_1$
        there exist distinct $u_2^{(1)}, \dots, u_2^{(k)} \in V_2$ having $(v_2,u_2^{(i)}) \in E_2$ and $G_1, u_1^{(i)} \sim^{L-1, c} G_2, u_2^{(i)}$;
        \item (back) for all distinct
        $u_2^{(1)}, \dots, u_2^{(k)} \in V_2$ such that $k \leq c$ and $(v_2,u_2^{(i)}) \in E_2$
        there exist distinct $u_1^{(1)}, \dots, u_1^{(k)} \in V_1$ having $(v_1,u_1^{(i)}) \in E_1$ and $G_1, u_1^{(i)} \sim^{L-1, c} G_2, u_2^{(i)}$;
    \end{enumerate}
    it being understood that $L \geq 0$, $c \geq 1$, and the back and forth conditions apply only when $L \geq 1$.
\end{definition}
We extend $c$-graded, $L$-turn bisimulations with unbounded global counting as follows.
\begin{definition}\label{def: bounded bisim}
    $G_1, v_1 \Rsim G_2, v_2$ when:
    \begin{enumerate}
        \item $G_1, v_1 \sim^{L, c} G_2, v_2$;
        \item for each $u \in V_1 \cup V_2$
        we have $|\{ u_1 \in V_1 \mid G_1, u_1 \sim^{L, c} G, u \}| = |\{ u_2 \in V_2 \mid G_2, u_2 \sim^{L, c} G, u \}|$,
    \end{enumerate}
    where $G$ is $G_1$ or $G_2$ depending on whether $u$ is in $V_1$ or $V_2$.
\end{definition}
Observe that $c$-graded, $L$-turn bisimulations with global counting are an upper bound in
terms of graph distinguishing power for $L$-layered ACR-GNNs with $c$-bounded aggregation functions
in a similar manner as the 1-Weisfeiler-Leman algorithm is for AC-GNNs \cite{xu2018howPowerful}.
\begin{restatable}{lemma}{lemmabisiminvaracr}\label{lemma_bisim_invar_acr}
    Suppose $G_1, v_1 \Rsim G_2, v_2$.
    Then, for each $L$-layered ACR-GNN $\gnn$ with $c$-bounded aggregation functions
    we have $\gnn(G_1, v_1) = \gnn(G_2, v_2)$.
\end{restatable}
Thus, supposing that $\varphi$ is an FO classifier captured by an ACR-GNN $\gnn$
with $c$-bounded aggregation functions, we must have, by
Lemma~\ref{lemma_bisim_invar_acr}, that $\varphi$ is invariant under $\Rsim$,
where $L$ is the number of layers of $\gnn$.

\begin{remark}
    Our result hinges on the fact that the $c$-graded bisimulations with global
    counting need only be considered up to $L$ turns (with $L$ being the number
    of layers a GNN has). When the turn bound is not present, properties outside
    of $\GML^\exists$ (e.g. strict linear orders) become invariant under the
    bisimulations in question, thus disallowing the desired characterisations.
\end{remark}

We now additionally define a variant of $\Rsim$ which can globally count only in
a bounded way. When given some $q' \geq 1$ we write $G_1, v_1 \Rsimq G_2, v_2$
for the bisimulation $\Rsim$ but with the following modification to Condition~2 in Definition~\ref{def: bounded bisim}:
\begin{enumerate}
        \item[] for each $u \in V_1 \cup V_2$ we have $\min(|\{ u_1 \in V_1 \mid
        G_1, u_1 \sim^{L, c} G, u \}|, q') = \min(|\{ u_2 \in V_2 \mid G_2, u_2
        \sim^{L, c} G, u \}|, q')$.
\end{enumerate}
The main part of the proof of
Proposition~\ref{proposition_hard_part_of_characterisation} is then to show that
a FO formula $\varphi$ that is invariant under $\Rsim$ is also invariant under
$\Rsimq$ for some $q'$.

We mention here that the traditional approach of showing this kind of result
involves unraveling and adding elements as seen in
\cite{otto2004Characterisations, otto2023gradedmodallogiccounting} and fails in
our setting as these operations produce featured graphs that are not bisimilar
via $\Rsim$ to their original counterparts.
Thus, instead of saturating the vertex set until a desired graph is obtained, we
opt to homogenise the graph in place without adding vertices.
For this we use two different $\Rsim$-preserving operations.
\begin{restatable}{lemma}{propfreeedgetransfer}\label{prop_free_edge_transfer}
    Taking $G, v$ to be a featured pointed graph
    suppose that $(v, w) \in E$, whilst $(v, w') \not\in E$ for some $w, w' \in V$
    having $G, w \sim^{L-1, c} G, w'$.
    Then, the graph $G'$ over the same vertices and feature vectors as $G$, and edges
    $E' = (E \setminus \{ (v, w) \}) \cup \{ (v, w') \}$
    is such that $G, u \Rsim G', u$ for all $u \in V$.
\end{restatable}
\begin{restatable}{lemma}{propfreewitness}\label{prop_free_witness}
    Taking $G, v$ to be a featured pointed graph
    suppose there are distinct $w^{(1)}, \dots, w^{(c)}, w' \in V$
    with $G, w^{(i)} \sim^{L-1, c} G, w'$
    and $(v, w^{(i)}) \in E$ for all $i \in [1, c]$, whilst $(v, w') \not\in E$.
    Then, the graph $G'$ over the same vertices and feature vectors as $G$, and edges
    $E' = E \cup \{ (v, w') \}$
    is such that $G, u \Rsim G', u$ for all $u \in V$.
\end{restatable}
\begin{example}
    The application of Lemmas~\ref{prop_free_edge_transfer}~and~\ref{prop_free_witness} with $L=1$ and $c=2$
    on the vertex $v$ in the left graph to form the right graph. By Lemma~\ref{lemma_bisim_invar_acr}
    each ACR-GNN with $2$-bounded aggregation functions and with $1$ layer will classify both graphs (pointed at $v$) the same way.
    
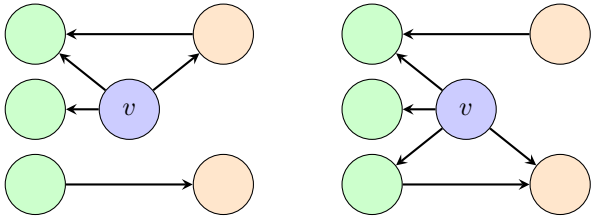
\begin{center}
        \begin{tikzpicture}[
    s_node/.style={circle, draw, fill=blue!20, minimum size=8mm},
    iota_node/.style={circle, draw, fill=green!20, minimum size=8mm},
    t_node/.style={circle, draw, fill=orange!20, minimum size=8mm},
    edge/.style={->, thick, >=stealth}
]
    \node[s_node] (v) at (0,-1) {$v$};

    \node[iota_node] (u1) at (-1.25, 0) {};
    \node[iota_node] (u2) at (-1.25, -1) {};
    \node[iota_node] (u3) at (-1.25, -2) {};

    \node[t_node] (w1) at (1.25, 0) {};
    \node[t_node] (w2) at (1.25, -2) {};

    \draw[edge] (v) -- (u1);
    \draw[edge] (v) -- (u2);

    \draw[edge] (v) -- (w1);

    \draw[edge] (u3) -- (w2);
    \draw[edge] (w1) -- (u1);
\end{tikzpicture} \hspace*{9mm}
        \begin{tikzpicture}[
    s_node/.style={circle, draw, fill=blue!20, minimum size=8mm},
    iota_node/.style={circle, draw, fill=green!20, minimum size=8mm},
    t_node/.style={circle, draw, fill=orange!20, minimum size=8mm},
    edge/.style={->, thick, >=stealth}
]
    \node[s_node] (v) at (0,-1) {$v$};

    \node[iota_node] (u1) at (-1.25, 0) {};
    \node[iota_node] (u2) at (-1.25, -1) {};
    \node[iota_node] (u3) at (-1.25, -2) {};

    \node[t_node] (w1) at (1.25, 0) {};
    \node[t_node] (w2) at (1.25, -2) {};

    \draw[edge] (v) -- (u1);
    \draw[edge] (v) -- (u2);
    \draw[edge] (v) -- (u3);

    \draw[edge] (v) -- (w2);

    \draw[edge] (u3) -- (w2);
    \draw[edge] (w1) -- (u1);
\end{tikzpicture}
    \end{center}
    Note that Lemma~\ref{prop_free_edge_transfer} is applied on $v$ with respect to yellow (right of $v$) nodes,
    whilst Lemma~\ref{prop_free_witness} is applied on $v$ with respect to green (left of $v$) nodes.
\end{example}
The main technique of the proof of
Proposition~\ref{proposition_hard_part_of_characterisation} is then as follows.
Let $q' \in \mathbb{N}$ be ``large enough'' with respect to the aggregation
bound $c$ and the quantifier rank of $\varphi$.
Suppose featured pointed graphs $G_1, v_1$ and $G_2, v_2$ are such that $G_1, v_1 \Rsimq G_2, v_2$ and $G_1, v_1 \models \varphi$.
We claim that it must be the case that $G_2, v_2 \models \varphi$.
Using Lemmas~\ref{prop_free_edge_transfer}~and~\ref{prop_free_witness}
we construct companion graphs $\hG_1$ and $\hG_2$
such that $G_i, v_i \Rsim \hG_i, v_i$ for both $i \in \{ 1, 2 \}$ and such
that $\hG_1, v_1$ and $\hG_2, v_2$ agree on all FO classifiers of the same quantifier rank as $\varphi$.
Then, using invariance of $\varphi$ under $\Rsim$ we have $\hG_1, v_1 \models \varphi$ and, by the above, $\hG_2, v_2 \models \varphi$.
Then, again by invariance of $\varphi$ under $\Rsim$, we must have $G_2, v_2 \models \varphi$.
Technical details on how to construct $\hG_1$ and $\hG_2$ are relegated to the appendix.

What we have now is that $\varphi$ is not only invariant under $\Rsim$, but $\Rsimq$ as well.
It is then easy to prove that $\Rsimq$ partitions graphs into finitely many equivalence classes,
each of which can be uniquely characterised by a $\GML^\exists$ formula.
Taking the disjunction of formula that entail $\varphi$ yields the $\GML^\exists$ classifier required by
Proposition~\ref{proposition_hard_part_of_characterisation}, thus completing the
proof of Proposition~\ref{theo_main_abstract_logic}.

It remains to argue how Proposition~\ref{theo_main_abstract_logic} can be used to
show Theorem~\ref{prop bounded graphs gml}. 
\begin{proof}[Proof sketch of Theorem~\ref{prop bounded graphs gml}]
    Let $c$ be the bound on the degree of vertices, and let $\psi_c$ be a
    $\GML^\exists$-formula that captures graphs of degree at most $c$.

First, let $\varphi$ be an FO classifier that is equivalent to a $\GML^\exists$ classifier.
Then, then $\varphi
\land \psi_c$ is equivalent to a $\GML^\exists$ classifier over all graphs.
Thus, by Proposition~\ref{theo_main_abstract_logic}, $\varphi \land \psi_c$ is
captured by an ACR-GNN with bounded aggregation functions over all graphs, and
thus $\varphi$ is captured by an ACR-GNN with bounded aggregation functions over all graphs of bounded degree.

For the other direction, let $\varphi$ be an FO classifier that is captured by
an ACR-GNN over all graphs of bounded degree. As ACR-GNNs can capture $\psi_c$,
there is an ACR-GNN that captures $\varphi \land \psi_c$ over all graphs.
Hence, by Proposition~\ref{theo_main_abstract_logic},
$\varphi \land \psi_c$ is equivalent over all graphs to a
$\GML^\exists$ classifier, which in turn implies that over graphs of
bounded degree, $\varphi$ is equivalent to a $\GML^\exists$ classifier.
\end{proof}




\section{Conclusion} \label{sec:conclusion}

We identified a bound for which (fragments of) $\Ctwo$ can be taken to characterise
ACR-GNNs relative to FO. We showed that even simple ACR-GNNs (i.e. those with sum aggregation and readout)
can express FO properties on both directed and undirected graphs that are outside of $\Ctwo$.
The technique used is a novel characterisation of linear orders in terms of number homomorphisms of length 2 paths
which might be of independent interest.

On the other side of the spectrum we showed that the FO classifiers captured
by ACR-GNNs are exactly those definable by $\GML^\exists$ formulas;
as long as the degree of graphs is bounded or the 
ACR-GNNs in question have bounded aggregation functions (but not necessarily bounded readout).
This result hinges on model theoretic techniques and might be of independent interest.

While these results improve our understanding of what kinds of ACR-GNNs can be
characterised in terms of $\Ctwo$, the exact expressive power of ACR-GNNs with
sum aggregation and readout remains an interesting question.
One, perhaps, interesting subcase is to consider the FO properties which are
expressible by ACR-GNNs that use unbounded sum aggregaton and bounded sum readout.
The question above is thus left for future research.


\section*{Acknowledgments}
Daumantas Kojelis was supported by DFG project LU 1417/4-1.

\bibliography{references}

\newpage

\appendix
\section{Proof of Lemma~\ref{lem:total-order-properties}}
\label{sec:sequence-proof}

\lemtotalorderproperties*

To show the lemma, we make heavy use of \emph{degree sequences} of graphs, and
their properties. Before we begin with the main proof, we first state some
standard properties of integer sequences, that we are going to apply to these
degree sequences.

With every sequence $a_1, \ldots, a_n$ of integers we associate its
\emph{conjugate sequence} $a'_1, \ldots, a'_n$, defined by
\[
    a'_i = |\{j \mid a_j \geq i \}|
\]
for every $i$ with $1 \leq i \leq n$.
A direct observation is that $\sum_{1 \leq i \leq n} a_i = \sum_{1 \leq i \leq n} a_i'$.

\begin{theorem}[Gale-Ryser Theorem]\label{thm:gale-ryser}
    Let $r_1, \ldots, r_n$ and $c_1, \ldots, c_n$ be two integer sequences with
    $r_1 \geq \cdots \geq r_n$.
    Then $r_1, \ldots, r_n$ and $c_1, \ldots, c_n$ are, respectively, the row and column sums
    of a binary\footnote{A matrix is binary if all of its entries are in $\{0, 1\}$. Note that
    an $n \times n$ matrix is binary if and only if it is an adjacency matrix of some graph on $n$ vertices.}
    $n \times n$ matrix if and only if
    \begin{itemize}
        \item $\sum_{i = 1}^n r_i = \sum_{i = 1}^n c_i$, and
        \item for every $k$ with $1 \leq k \leq n$, $\sum_{i = 1}^k r_i \leq \sum_{i = 1}^k c'_i$,
    \end{itemize}
    where $c'_1, \ldots, c'_n$ is the conjugate sequence of $c_1, \ldots, c_n$.
\end{theorem}
Note that it follows from the first point that the inequality in the second point
must be an equality for $k = n$. 

\begin{lemma}[Rearrangement Inequality]\label{lem:rearrange}
    For all number sequences $x_1, \ldots, x_n$ and $y_1, \ldots, y_n$ with $x_1 \geq
    \cdots \geq x_n$ and $y_1 \geq \cdots \geq y_n$, and all permutations $\sigma$ of $1, \ldots, n$,
    \[
        x_1 y_n + \cdots + x_n y_1 \leq x_1 y_{\sigma(1)} + \cdots + x_n y_{\sigma(n)}.
    \]
    Furthermore, if $x_1 > \cdots > x_n$, then for all permutations $\sigma$ that are not the identity:
    \[
        x_1 y_n + \cdots + x_n y_1 < x_1 y_{\sigma(1)} + \cdots + x_n y_{\sigma(n)}.
    \]
\end{lemma}

\begin{lemma}[Summation By Parts]\label{lem:summation-by-parts}
    Let $a_1, \ldots, a_n$ and $b_1, \ldots, b_n$ be number sequences.  Then
    \[
        \sum_{i = 1}^n a_i b_i = a_n B_n + \sum_{i = 1}^{n - 1} (a_{i + 1} - a_i) (- B_i)
    \]
    where $B_k = \sum_{i = 1}^k b_i$.
\end{lemma}

Our main tool to show Lemma~\ref{lem:total-order-properties} is now the
observation that a sequence of integers $a_1, \ldots, a_n$ that  sums to
$\binom{n}{2}$ satisfies $\sum_{i = 1}^{n} a_i a'_{n + 1 - i} = \binom{n}{3}$
if and only if it is $n - 1, \ldots, 0$.

\begin{lemma}\label{lem:sequence-step}
    Let $a_1, \ldots, a_n$ be a sequence of integers with $a_1 \geq \cdots \geq a_n$.
    Then,
    \begin{equation*}
        \sum_{i = 1}^n a_i = \binom{n}{2} \text{ implies } \sum_{i = 1}^{n} a_i a'_{n + 1 - i} \geq \binom{n}{3},
    \end{equation*} 
    with equality holding if and only if $a_1, \ldots, a_n$ is the sequence $n - 1, \ldots, 0$.
\end{lemma}

\begin{proof}
First, if $a_1, \ldots, a_n = n-1, \ldots, 0$, then $a'_1, \ldots, a'_n = n - 1,
\ldots, 0$ and thus
\[
    \sum_{i = 1}^{n} a_i a'_{n + 1 - i} = \sum_{i = 1}^n (n - i)(i - 1) = \binom{n}{3}.
\]

If $a_1, \ldots, a_n \neq n-1, \ldots, 0$, we show that there exists a sequence $b_1,
\ldots, b_n$ with $b_1 \geq \cdots \geq b_n$ and $\sum_{i = 1}^n b_i =
\binom{n}{2}$ such that
\[
    \sum_{i = 1}^n b_i b'_{n + 1 - i} < \sum_{i = 1}^n a_i a'_{n + 1 - i}.
\]
This suffices to show the lemma as, whenever $a_1, \ldots, a_n \neq n - 1,
\ldots, 0$, we can replace the sequence with $b_1, \ldots, b_n$ and strictly
decrease $\sum_{i = 1}^n a_i a'_{n + 1 - i}$. As this can only happen a finite
number of times, this process must arrive at the sequence $n - 1, \ldots, 0$
with $\sum_{i = 1}^n a_i a'_{n + 1 - i} = \binom{n}{3}$.

To show that such a sequence $b_1, \ldots, b_n$ exists, consider 
$\sum_{i = 1}^n a_i - (n - i)$. As $\sum_{i = 1}^n a_i = \binom{n}{2}$ and
$\sum_{i = 1}^n (n - i) = \binom{n}{2}$, it must be that
$\sum_{i = 1}^n a_i - (n - i) = 0$. However, as $a_1, \ldots, a_n \neq n-1,
\ldots, 0$, there must be some $i$ such that $a_i > (n - i)$ and some $i$ such
that $a_i < (n - i)$.
Let
\[
    p = \max \{ i \mid a_i > (n - i) \} \ \ \text{and} \ \ q = \min\{i \mid a_i < (n - i)\}.
\]
Clearly, $p \neq q$, $a_q < n - q$ and $a_p > n - p$.
Now set, for all $i$ with $1 \leq i \leq n$,
\[
    b_i = \begin{cases}
        a_i - 1 & i = p \\
        a_i + 1 & i = q \\
        a_i & \text{ otherwise.}
    \end{cases}.
\]
We show that $b_1, \ldots, b_n$ satisfies the desired properties. 
Clearly, as we added $1$ to some $a_q$ and subtracted $1$ from $a_p$,
$\sum_{i = 1}^n a_i = \sum_{i = 1}^n b_i$.

We show that $b_1 \geq \cdots \geq b_n$.
Assume for contradiction that there are $b_i, b_{i + 1}$ with $b_i < b_{i + 1}$.
As we know $a_i \geq a_{i + 1}$, we must be in one of the following cases:
\begin{itemize}
    \item $a_i = a_p$ and $a_i = a_{i + 1}$. But then $a_i > (n - i)$ and $a_{i
    + 1} > (n - (i + 1))$, contradicting that $p$ is maximal.
    \item $a_{i + 1} = a_q$ and $a_i = a_{i + 1}$. But then $a_{i + 1} < (n - (i
    + 1))$ and $a_{i} < (n - i)$, contradicting that $q$ is minimal.
    \item $a_i = a_p$, $a_{i + 1} = a_q$ and $a_i = a_{i + 1} + 1$. Then,
    $a_i > (n - i)$ and $a_{i + 1} < (n - (i + 1))$, meaning that 
    $a_{i + 1} + 1$  must be both larger and smaller than $(n - 1)$, a contradiction.
\end{itemize}
Therefore $b_1 \geq \cdots \geq b_n$.

As the last step, we now show that 
\[
    \sum_{i = 1}^n b_i b'_{n + 1 - i} < \sum_{i = 1}^n a_i a'_{n + 1 - i}.
\]
Consider the conjugate sequence $b'_1, \ldots, b'_n$. By definition of the
sequence $b_1, \ldots, b_n$, there is one less $b_i$ with $b_i \geq a_p$, and
one more $b_i$ with $b_i \geq (a_q + 1)$. Therefore,
 \[
    b'_i = \begin{cases}
        a'_i & i = a_p = a_q + 1 \\
        a'_i - 1 & i = a_p  \text{ and } i \neq a_q + 1\\
        a'_i + 1 & i = a_q + 1  \text{ and } i \neq a_p \\
        a'_i & \text{ otherwise.}
    \end{cases}
\]
Therefore, $b_i b'_{n + 1 - i} = a_i a'_{n + 1 - i}$ for all $i \notin \{p, q, a_p, a_q + 1\}$.
It follows that
\begin{align*}
    \sum_{i = 1}^n  b_i b'_{n + 1 - i} = \sum_{i = 1}^n a_i a'_{n + 1 - i} &+ a'_{n + 1 - q} - a'_{n + 1 - p} \\
    &+ b_{n - a_q} - b_{n + 1 - a_p}.
\end{align*}
We now make the following observations to bound this difference:
\begin{enumerate}
    \item $a'_{n + 1 - p} = p$. This is because by choice of $p$, $a_p \geq (n +
    1 - p)$, and $a_i < (n + 1 - p)$ for all $i > p$, and $a_i \geq a_p$ for all
    $i \leq p$.
    \item $a'_{n + 1 - q} = q - 1$. This is because by choice of $q$, all $a_i$
    with $i < q$ satisfy $a_i \geq (n - i) \geq (n + 1 - q)$, and all $a_i$ with
    $i \geq q$ satisfy $a_i < (n + 1 - q)$.
    \item $b_{n - a_q} \leq a_q$. This is because $a_q \leq (n - q - 1)$ by choice
    of $q$, and hence $b_{n - a_q} \leq b_{n - (n - q - 1) } = b_{q + 1}$, which
    satisfies $b_{q + 1} \leq a_q$ by construction of the sequence $b_1, \ldots, b_n$.
    \item $b_{n + 1 - a_p} \geq a_p - 1$. This is because $a_p \geq (n - p + 1)$ by choice of $p$, and hence
    $b_{n + 1 - a_p} \geq b_{n + 1 - (n - p + 1)} = b_{p} = a_{p - 1}$.
\end{enumerate}
Using these observations,
\[
a'_{n + 1 - q} - a'_{n + 1 - p} + b_{n - a_q} - b_{n + 1 - a_p } \leq q - p + a_q - a_p
\]
and
\[
q - p + a_q - a_p < q - p + n - q - n - p = 0.
\]
Therefore,
\[
\sum_{i = 1}^n  b_i b'_{n + 1 - i} < \sum_{i = 1}^n a_i a'_{n + 1 - i}
\]
as required.
\end{proof}

With these tools in hand, we can now show Lemma~\ref{lem:total-order-properties}.

\begin{proof}[Proof of Lemma~\ref{lem:total-order-properties}]
Let $G = (V, E)$ be a graph with $n$ vertices that satisfies both $|E| =
\binom{n}{2}$ and $|\hom(P_2, G)| = \binom{n}{3}$.
Let $v_1, \ldots, v_n$ be an enumeration of the vertices in $V$ such that
$|N^G_{\text{out}}(v_1)| \geq \cdots \geq |N^G_{\text{out}}(v_n)|$.
For readability, set $a_i = |N^G_{\text{out}}(v_i)$ and $b_i = |N^G_{\text{in}}(v_i)$ for $1 \leq i \leq n$.
We now aim to show that the degree sequence $(a_1, b_1), \ldots, (a_n, b_n)$ must be
$(n - 1, 0), \ldots, (0, n - 1)$. This suffices to show the lemma, as the strict
linear order with $n$ vertices is the only graph (up to isomorphism) with this
degree sequence. 

Observe that the homomorphisms in $\hom(P_2, G)$ correspond exactly to
combinations of incoming and outgoing edge around a middle vertex. Hence,
\[
    |\hom(P_2, G)| = \sum_{i = 1}^n a_i b_i = \binom{n}{3}.
\]
Our aim is now to manipulate the sum $\sum_{i = 1}^n a_i b_i$ until we can apply
Lemma~\ref{lem:sequence-step} to show that $a_1, \ldots, a_n$ must be the desired sequence.
Now let $c_1, \ldots, c_n$ be a permutation of $b_1, \ldots, b_n$, such that
$c_1 \geq \cdots \geq c_n$. By the rearrangement inequality (Lemma~\ref{lem:rearrange}),
\begin{equation}
    \binom{n}{3} = \sum_{i = 1}^n a_i b_i \geq \sum_{i = 1}^n a_i c_{n + 1 - i}. \label{eq:first-hom}
\end{equation}
As we can view the sequences $c_1, \ldots, c_n$ and $a_1, \ldots, a_n$ as the
row and column sums of a suitably permutated adjacency matrix of $G$, the
Gale-Ryser theorem (Theorem~\ref{thm:gale-ryser}) yields
\begin{align}
    \sum_{i = 1}^n c_i &= \sum_{i = 1}^n a'_i \quad \text{and} \notag{} \\ 
    \sum_{i = 1}^k c_i &\leq \sum_{i = 1}^k a'_i \quad \text{for every } k = 1, \ldots, n \label{eq:gale-ryser2}
\end{align}
where $a'_1, \ldots, a'_n$ is the conjugate sequence of $a_1, \ldots, a_n$.
This allows us to further analyse the sum $\sum_{i = 1}^n a_i c_{n + 1 - i}$. We
start with the observation that through reindexing
\[
    \sum_{i = 1}^n a_i (c_{n + 1 - i} - a'_{n + 1 - i}) = \sum_{i = 1}^n a_{n + 1 - i} (c_i - a'_i).
\]
Applying summation by parts (Lemma~\ref{lem:summation-by-parts}) to the above yields that
\begin{align}
    \sum_{i = 1}^n a_i &(c_{n + 1 - i} - a'_{n + 1 - i})  \notag{} \\
    & = a_1 R_n + \sum_{i = 1}^{n - 1} (a_{n - i} - a_{n + 1 - i}) (- R_i) \label{eq:some-step}
\end{align}
where $R_i = \sum_{j = 0}^i (c_i - a'_i)$.
By the inequalities~\eqref{eq:gale-ryser2}, $R_i \leq 0$ for all $i$ with $1
\leq i \leq n$ and, as for $k = n$ the inequality~\eqref{eq:gale-ryser2} is an
equality, $R_n = 0$. Additionaly, by the initial ordering of the $a_i$, it holds that
$a_{n - i} \geq a_{n + 1 - i}$ for all $i$. Hence all parts on the right side of
the equality~\eqref{eq:some-step} are non-negative, and therefore
\[
\sum_{i = 1}^n a_i (c_{n + 1 - i} - a'_{n + 1 - i}) \geq 0
\] and, by splitting this sum
\begin{equation}
    \sum_{i = 1}^n a_i c_{n + 1 - i} \geq \sum_{i = 1}^n a_i a'_{n + 1 - i}. \label{eq:second-hom}
\end{equation}
As $G$ satisfies $|E| = \binom{n}{2}$ it follows that $\sum_{1 \leq i \leq n}
a_i = \binom{n}{2}$. An application of Lemma~\ref{lem:sequence-step} to the
sequence $a_1, \ldots, a_n$ then yields 
\[
    \sum_{i = 1}^n a_i a'_{n + 1 - i} \geq \binom{n}{3}.
\]
This implies that \eqref{eq:first-hom} and \eqref{eq:second-hom} must be
equalities, in turn yielding
\[
    \sum_{i = 1}^n a_i a'_{n + 1 - i} = \binom{n}{3}.
\]
It then follows from Lemma~\ref{lem:sequence-step} that and $a_1, \ldots, a_n =
n - 1, \ldots, 0$.

It now remains to show that $b_1, \ldots, b_n = 0, \ldots, n - 1$.
Due to \eqref{eq:second-hom} being an equality, it must be that
\[
    a_n R_n + \sum_{i = 1}^{n - 1} (a_{n - i} - a_{n + 1 - i}) (-R_i) =  0.
\]
As $a_1, \ldots, a_n = n - 1, \ldots, 0$, all differences $a_{n - 1} - a_{n + 1
- 1}$ are $1$ and we get
\[
    \sum_{i = 1}^{n - 1} (- R_i ) =  0.
\]
By \eqref{eq:gale-ryser2}, $R_i \leq 0$ for every $i$ with $1 \leq i \leq n$. The
only possibility for this is that $R_i = 0$ for all $i$, and hence
\[
    c_1, \ldots, c_n = a'_1, \ldots, a'_n = a_1, \ldots, a'_n = n - 1, \ldots, 0.
\]
As $a_1 > \cdots > a_n$ and \eqref{eq:first-hom} is an equality, the
rearrangement inequality (Lemma~\ref{lem:rearrange}) implies that $b_1, \ldots,
b_n = 0, \ldots, n - 1$, as desired.
\end{proof}

\section{Proof of Proposition~\ref{prop_gadgetisations_expressible}}

\propgadgetisationsexpressible*

We aim to prove that the class of gadgetisations of linear-orders is captured by an FO classifier and a simple ACR-GNN.
We do so by splitting our endeavour into two lemmas:
\begin{lemma}\label{lemma_fo_captures_gad_lin_ord}
    There is an FO classifier which captures gadgetisations of linear orders.
\end{lemma}
\begin{proof}
    We proceed by defining a series of FO formulas over unary $P_1, P_2$ and binary $E$ symbols which capture the required property.
    The formula $\bigwedge_{i=1}^4 \psi_i$ below is in first-order logic
    and is only satisfied by gadgetisations of a directed graph.
    The first formula, $\psi_1$, partitions the universe into three disjoint sets:
    \begin{equation*}
        \forall x \Big( \neg P_1(x) \vee \neg P_2(x) \Big).
    \end{equation*}
    It is to be understood that vertices satisfying $P_1$ are members of
    $\tV_s$; vertices satisfying $P_2$ are members of $\tV_t$; and vertices
    satisfying neither are members of $\tV_{id}$.

    Then, $\psi_2$ requires that edges do only exist between $\tV_s$ and $\tV_t$; $\tV_s$ and $\tV_{id}$; $\tV_t$ and $\tV_{id}$:
    \begin{equation*}
        \forall x y \Big( E(x, y) \to \alpha(x, y) \Big),
    \end{equation*}
    where $\alpha$ is the criterion enforcing bipartiteness and lack of edges between elements in $\tV_{id}$.
    The requirement can be written as the disjunction of the following:
    \begin{equation*}
        \begin{split}
        P_1(x) \wedge P_2(y),&\; P_2(x) \wedge P_1(y), \\
        P_1(x) \wedge \neg P_1(y) \wedge \neg P_2(y),&\; P_2(x) \wedge \neg P_1(y) \wedge \neg P_2(y),\\
        P_1(y) \wedge \neg P_1(x) \wedge \neg P_2(x),&\; P_2(y) \wedge \neg P_1(x) \wedge \neg P_2(x).
        \end{split}
    \end{equation*}

    Now, $\psi_3$ requires that vertices in $\tV_s \cup \tV_t$ be connected to exactly one element in $\tV_{id}$:
    \begin{equation*}
        \forall x \Big( P_1(x) \vee P_2(x) \to \exists^{=1} y \big( E(x, y) \wedge \neg P_1(y) \wedge \neg P_2(y) \big) \Big).
    \end{equation*}

    Lastly, $\psi_4$ requires that elements of $\tV_{id}$ be paired with exactly one element of $\tV_s$ and exactly one element of $\tV_t$:
    \begin{equation*}
        \forall x \Big( \neg P_1(x) \wedge \neg P_2(x) \to \bigwedge_{i =1}^2 \exists^{=1} y \big( E(x, y) \wedge P_i(y) \big) \Big).
    \end{equation*}

    It is easy to verify that a 2-featured pointed undirected graph $(G, v)$ satisfies
    $\bigwedge_{i=1}^4 \psi$ if and only if $(G, v)$ is a gadgetisation of a
    directed graph.

    The formula $\bigwedge_{i=1}^4 \varphi_i$ below, when taken together with $\bigwedge_{i=1}^4 \psi_i$,
    makes sure that satisfying graphs are gadgetisations of a linear order.
    Writing $\varepsilon(x, y, z)$ for the formula
    \begin{equation*}
         P_1(x) \wedge (\neg P_1(y) \wedge \neg P_2(y)) \wedge P_2(z) \wedge E(x, y) \wedge E(y, z)
    \end{equation*}
    we capture irreflexivity with $\varphi_1$:
    \begin{equation*}
        \forall s \iota t \Big( \varepsilon(s, \iota, t) \to \neg E(s, t) \Big).
    \end{equation*}

    With the sentence $\varphi_2$ we capture asymmetry:
    \begin{align*}
        \forall s\iota t s'\iota't' \Big( \varepsilon(s,\iota,t) \wedge \varepsilon(s', \iota', t') \wedge E(s, t') \to \neg E(s', t) \Big).
    \end{align*}

    The sentence $\varphi_3$ specifies transitivity:
    \begin{align*}
        \forall s \iota t s'\iota't' s''\iota''t'' \Big( \varepsilon(s,\iota,t) \wedge \varepsilon(s', \iota', t') \wedge
        \varepsilon(s'', \iota'', t') \wedge \\ E(s, t') \wedge E(s', t'') \to E(s, t'') \Big).
    \end{align*}

    And, lastly, $\varphi_4$ requires totality:
    \begin{align*}
        \forall s\iota t s'\iota't' \Big( \varepsilon(s,\iota,t) \wedge &\varepsilon(s', \iota', t') \to\\ &E(s, t') \vee E(s', t) \vee \iota = \iota' \Big).
    \end{align*}

    It is then easy to verify that a 2-featured (possibly pointed) undirected graph satisfies $\bigwedge_{i=1}^4 \varphi_i \wedge \psi_i$
    if and only if it is a gadgetisation of a linear order.
\end{proof}

\begin{lemma}
    There is a simple ACR-GNN that captures gadgetisations of linear orders.
\end{lemma}
\begin{proof}
    Take $\psi_1, \dots, \psi_4$ as defined in Lemma~\ref{lemma_fo_captures_gad_lin_ord}.
    It is easy to verify that the above FO classifier is equivalent to a $\GML^{\exists}$ classifier.
    Thus, using Proposition~\ref{proposition_easy_part_of_characterisation}, it is easy to translate them to a simple ACR-GNN $\gnn$
    such that $G, v \models \bigwedge_{i=1}^4 \psi_i$ if and only if $\gnn(G, v) = 1$.
    Thus, we can restrict our attention to gadgetisations of directed graphs.

    Note that the formulas $\varphi_1, \ldots, \varphi_2$ do not correspond to
    $\GML^{\exists}$ formulas in an obvious way, and hence
    Proposition~\ref{proposition_easy_part_of_characterisation} provides no direct translation of $\varphi_1, \dots, \varphi_4$
    into simple ACR-GNNs. We instead rely on our characterisation of linear order in terms of homomorphisms of $P_2$, or, rather,
    gadgetisations $\widetilde{P}_2$ of $P_2$. To this end we claim that, for any directed graph $G = (V, E)$ and its gadgetisation $\tG$:
    \begin{itemize}
        \item $|\tE| = |E| + 2|V|$, and
        \item $|\hom(P_2, G)| = |\hom(\widetilde{P}_2, \tG)|$.
    \end{itemize}
    The first point is immediate by the definition of a gadgetisation (Definition~\ref{def:gadget}).
    For the second point we show that there is a bijection $\pi$ between the subgraphs of $G$ that are isomorphic to $P_2$ and
    the subgraphs of $\tG$ that are isomorphic to $\widetilde{P}_2$.
    Specifically, if $H$ is a subgraph of $G$ isomorphic to $P_2$, then 
    we set $\pi(H) = \widetilde{H}$.
    It is easy to verify that $\widetilde{H}$ is a subgraph of $\tG$.
    For injectivity suppose that there are subgraphs $H, F$ of $G$ that are isomorphic to $P_2$
    with $\pi(H) = \pi(F)$. Thus, $\widetilde{H} = \widetilde{F}$.
    But then the gadgetisations both have vertices $ \bigcup_{x \in \{v, u, w \}} \{ s_x, \iota_x, t_x \}$
    and edges $\{ (s_v, t_u), (s_u, t_w) \}$
    for some $v, u, w \in V$. Thus, by definition, $H$ and $F$ must both have vertices $v, u, w$ and edges $(v, u)$ and $(u, w)$.
    Thus, $H = F$. Surjectivity follows from similar observations.

    The required simple ACR-GNN is then obtained by
    applying the obvious modifications to Lemma~\ref{lemma_simple_gnn_directed_orders}.
\end{proof}

\section{Proof of Proposition~\ref{proposition_C2_bisimulation}}

\propositioncbisimulation*

We start by giving the appropriate notion of bisimulations in regards to $\Ctwo$
on featured pointed undirected graphs.
Taking $G_i = (V_i, E_i, f_i)$ for $i \in \{ 1, 2\}$ we define:
\begin{definition}
    $G_1, v_1, \equiv^{L,c} G_2, v_2$ when:
    \begin{enumerate}
        \item $f_1(v_1) = f_2(v_2)$;
        \item (forth) for every distinct $u_1^{(1)}, \dots, u_1^{(k)} \in V_1$ with $k \leq c$
        there is some distinct $u_2^{(1)}, \dots, u_2^{(k)} \in V_2$ s.t. $(v_1,u_1^{(i)}) \in E_1$ iff $(v_2,u_2^{(i)}) \in E_2$, and $G_1, u_1^{(i)} \equiv^{L-1, c} G_2, u_2^{(i)}$;
        \item (back) for every distinct $u_2^{(1)}, \dots, u_2^{(k)} \in V_2$  with $k \leq c$
        there is some $u_1^{(1)}, \dots, u_1^{(k)} \in V_1$ s.t. $(v_1,u_1^{(i)}) \in E_1$ iff $(v_2,u_2^{(i)}) \in E_2$, and $G_1, u_1^{(i)} \equiv^{L-1, c} G_2, u_2^{(i)}$;
    \end{enumerate}
    it being understood that $L \geq 0$, $c \geq 1$ and the back and forth conditions apply only when $L \geq 1$.
\end{definition}

In the literature, the above relation is commonly referred to as the $L$-turn 2-pebble $c$-bijective (or $c$-counting) back-and-forth equivalence~\cite{Hella96}.
The following is well-known:
\begin{lemma}
    Suppose that $\mathcal{P}$ is some property of featured pointed undirected graphs.
    Then, the following are equivalent:
    \begin{itemize}
        \item There is no $\Ctwo$ classifier that captures $\mathcal{P}$;
        \item for each $L, c \in \mathbb{N}$ there is a pair of pointed undirected graphs $G_1, v_1$ and $G_2, v_2$
            such that $(G_1, v_1) \in \mathcal{P}$ whilst $(G_2, v_2) \not\in \mathcal{P}$, but
            $G_1, v_1 \equiv^{L, c} G_2, v_2$.
    \end{itemize}
\end{lemma}
Thus, to prove Proposition~\ref{proposition_C2_bisimulation} we need only the following:
\begin{lemma}
    For each $L, c \geq 0$ there is a gadgetisation $G = (V, E, f)$ of a strict
    linear order, and a gadgetisation $H = (V, F, f)$ of a graph that is not a
    strict linear order such that $G, v \equiv^{L, c} H, v$ for all $v \in V$.
\end{lemma}
\begin{proof}
    Let $n := Lc{+}1$ and define $G$ to be the gadgetisation of a strict linear order
    of size $2n{+}1$.
    We segregate vertices $V$ into $V_{s}$, $V_{t}$ and $V_{id}$
    in accordance to their function.
    Additionally, we assume that vertices of $V_{s}$ are enumerated as $s_{-n}, \dots, s_0, \dots, s_n$,
    it being understood that $s_{i}$ corresponds to the $(i{+}n{+}1)$-th (smallest) element in the original strict linear order.
    We assume that an analogous enumeration is fixed for $V_{t} = \{ t_{-n}, \dots, t_n \}$ and
    $V_{id} =  \{ \iota_{-n}, \dots, \iota_n \}$.
    Then, $H$ is defined to be a copy of $G$ but with the following edge modification:
    $F = \big(E \setminus \{ (s_{-1}, t_1) \} \big) \cup \{ (s_{1}, t_{-1} ) \}$.
    It is easy to verify that $H$ is a gadgetisation of a graph that is not a linear order
    (specifically, the graph from which $H$ is formed contains a cycle of length 3 ``in the middle'').
    We proceed by showing that $G, v \equiv^{L, c} F, v$ by induction on $k \leq L$.
    In particular, we maintain the following at each step:
    \begin{enumerate}
        \item $G, v \equiv^{k, c} H, v$;
        \item $G, s_i \equiv^{k, c} H, s_j$;
        \item $G, t_i \equiv^{k, c} H, t_j$; and
        \item $G, \iota_i \equiv^{k, c} H, \iota_j$
    \end{enumerate}
    for all $i, j \in [{-}n{+}ck, n{-}ck]$ and all $v \in V$.

    The base case $k=0$ is immediate as both $G$ and $H$ share the same feature function $f$ and set of nodes $V$.
    For the inductive step suppose the above conditions are satisfied for $k < L$. 
    We show that the same holds for $k{+}1$ in place of~$k$.
    For this purpose take any $v \in V$ and consider the following cases.

    \paragraph*{When $v = s_i$ with $i \in [{-}n{+}c(k{+}1), n{-}c(k{+}1)]$.}
    We aim to show that 
    \begin{equation*}
        G, s_i \equiv^{k+1, c} H, s_j \text{ for } j \in [{-}n{+}c(k{+}1), n{-}c(k{+}1)].
    \end{equation*}
    To achieve this, we need to satisfy the back-and-forth requirements of the bisimulation relation.
    We will show how to establish forth whilst noting that the back direction is symmetric.
    Thus, take, by forth, some
    $s_{x_1}, \dots, s_{x_X}, t_{y_1}, \dots, t_{y_Y}, \iota_{z_1}, \dots, \iota_{z_Z}$ from $V$
    with $X + Y + Z = c$. We now find $\equiv^{k, c}$-bisimilar matches for these elements.

    Let us first inspect $s_{x_1}, \dots, s_{x_X}$.
    Since both graphs are bipartite in restriction to $V_{s} \cup V_{t}$,
    we have that neither $(s_i, s_{x_p}) \in E$ nor $(s_j, s_{x_p}) \in F$ for any choice of $p \in [1, X]$.
    Immediately, by our induction hypothesis, $G, s_{x_p} \equiv^{k, c} H, s_{x_p}$.
    Thus when presented with $s_{x_1}, \dots, s_{x_X}$ we choose the same elements as the answer to the forth challenge.

    Now, let us take $t_{y_1}, \dots, t_{y_Y}$ into consideration.
    Taking any $t_{y_p}$ with $p \in [1, Y]$ we see that 
    $y_p \not\in [{-}n{+}ck, n{-}ck]$ implies that
    $(s_i, t_{y_p}) \in E$ if and only if $(s_j, t_{y_p}) \in F$.
    Thus, for these specific $t_{y_p}$
    we choose the same elements $t_{y_p}$ themselves as their $\equiv^{k, c}$-bisimilar matches
    (as, by I.H., $G, t_{y_p} \equiv^{k, c} H, t_{y_p}$).
    On the other hand, suppose $y_p \in [{-}n{+}ck, n{-}ck]$.
    In case $(s_i, t_{y_p}) \not\in E$
    pick any $y_p' \in [{-}n{+}ck, {-}n{+}c(k{+}1))$.
    It is easy to verify that $(s_j,t_{y_p'}) \not\in F$
    and, by I.H., that $G, t_{y_p} \equiv^{k, c} H, t_{y_p'}$.
    In fact, there are at least $c$ choices of $y_p'$.
    Thus, we can pick a distinct $y_p'$ for each $y_p$ of the sort above.
    The reasoning when $(s_i, t_{y_p}) \in E$ is similar, only then we pick
    $y_p' \in (n{-}c(k+1), n{-}ck]$.

    Lastly, take $\iota_{z_1}, \dots, \iota_{z_Z}$ and let $p \in [1, Z]$.
    Notice that $(s_i, \iota_{z_p}) \in E$
    if and only if $z_p = i$.
    If such a $p$ exists, then
    we answer with $\iota_{z_p'}$ having $z_p' = j$ as this is the only element for which
    $(s_j, \iota_{z_p'}) \in F$.
    In the case where $z_p \neq i$ the required result is achieved by picking $z_p' = z_p$ when $z_p \not\in [{-}n{+}ck, n{-}ck]$
    and any $z_p' \in [{-}n{+}ck, n{-}ck] \setminus [{-}n{+}c(k{+}1), n{-}c(k{+}1)]$
    when $z_p \in [{-}n{+}ck, n{-}ck]$. It is easy to verify that there are enough distinct
    choices to match every $\iota_{z_1}, \dots, \iota_{z_Z}$.
    Putting the above together we have, by I.H., that $G, \iota_{z_p} \equiv^{k, c} H, \iota_{z_p'}$
    thus concluding this section of our proof.

    \paragraph*{When $v = s_i$ with $i \not\in [{-}n{+}c(k{+}1), n{-}c(k{+}1)]$.}
    Notice that $-1, 0, 1 \in [{-}n{+}c(k{+}1), n{-}c(k{+}1)]$
    as $n{-}c(k{+}1) = cq{+}1{-}c(k{+}1) \geq  cq{+}1{-}cq = 1$.
    Similarly, ${-}n{+}c(k{+}1) \leq -1$.
    It is easy to establish $G, s_i \equiv^{k{+}1, c} H, s_i$
    by noting that $(s_i, u) \in E$ iff $(s_i, u) \in F$
    and that, by I.H., $G, u \equiv^{k, c} H, u$ for each $u \in V$.

    \paragraph*{Other cases.} Similar reasoning applies when $v = t_i$ or $v = \iota_i$.
\end{proof}
\section{Proof of Proposition~\ref{proposition_hard_part_of_characterisation}}

\propositionhardpartofcharacterisation*

We prove Proposition~\ref{proposition_hard_part_of_characterisation}
in two stages. First, we will show the following weaker proposition:
\begin{proposition}\label{proposition_hard_part_of_characterisation_weak}
    Suppose an ACR-GNN $\gnn$ with $c$-bounded aggregation functions captures an FO classifier $\varphi$ of quantifier rank $q$.
    Then, $\varphi$ is invariant under $\Rsimq$, where $L$ is the number of layers of $\gnn$ and $q' = q+c-1$.
\end{proposition}
Then, having the above, we will show how to construct a formula in $\GML^\exists$ that captures the classifier $\varphi$
as required by Proposition~\ref{proposition_hard_part_of_characterisation}.

For the rest of the section let us keep the FO classifier $\varphi$, and constant $c \geq 1$ fixed.
The high level overview for the proof of Proposition~\ref{proposition_hard_part_of_characterisation_weak}
is as follows:
\begin{enumerate}
    \item Suppose $\varphi$ is captured by an ACR-GNN $\gnn$ with $c$-bounded aggregation functions.
        Then, $\varphi$ must be invariant under $\Rsim$, where $L$ is the number of layers in $\gnn$
        (Lemma~\ref{lemma_bisim_invar_acr}).
    \item Every graph $G = (V, E, f)$ has a companion graph $\hG = (V, \widehat{E}, f)$
        such that $G, u \Rsim \hG, u$ for all $u \in V$ and in which elements behave ``predictably''.
        Thus, $G, v \models \varphi$ iff $\hG, v \models \varphi$
        (Lemmas~\ref{prop_free_edge_transfer},~\ref{prop_free_witness},~\ref{lemma_initial_good_graph},~and~\ref{lemma_saturation}).
    \item Having $\varphi$ invariant under $\Rsim$ we obtain invariance under $\Rsimq$ as follows.
        Taking any $(G_1, v_1)$, $(G_2, v_2)$ assume that $G_1, v_1 \Rsimq G_2, v_2$, it being understood that $G_1, v_1 \models \varphi$.
        We produce a companion graph $\hG_1$ for $G_1$ such that $\hG_1, v_1 \models \varphi$.
        Having this, we carefully construct a companion graph $\hG_2$ for $G_2$ (Lemma~\ref{lemma_homogeneity})
        such that $\hG_1, v_1$ and $\hG_2, v_2$ agree on all FO classifiers of quantifier rank $q$.
        We conclude that $\hG_2, v_2 \models \varphi$ and $G_2, v_2 \models \varphi$ (Lemma~\ref{lemma_upgrading}).
\end{enumerate}

We start by showing that $c$-graded, $L$-turn bisimulations with global counting
are the upper-bound for $L$-layered ACR-GNNs with $c$-bounded aggregation functions
in terms of their distinguishing power.
\lemmabisiminvaracr*
\begin{proof}
    We proceed by induction on $i \leq L$ concerning ourselves with the vertex embeddings $f_1^{(i)} \colon V_1 \to \mathbb{R}^{d_i}$
    and $f_2^{(i)} \colon V_2 \to \mathbb{R}^{d_i}$ that are computed by the $i$-th layer of $\gnn$ running on $G_1$ and $G_2$, respectively.
    We show that
    \begin{equation*}
        G_1, v_1 \RsimP{i} G_2, v_2 \text{ implies } f_1^{(i)}(v_1) = f_2^{(i)}(v_2)
    \end{equation*}
    for all $v_1 \in V_1$ and $v_2 \in V_2$.
    The base case $i = 0$ is immediate as $G_1, v_1 \RsimP{0} G_2, v_2$ 
    implies that $f_1(v_1)  = f_2(v_2)$, and therefore  $f_1^{(0)}(v_1) = f_2^{(0)}(v_2)$.
    Now, suppose that $G_1, v_1 \RsimP{i+1} G_2, v_2$.
    Clearly, $G_1, v_1 \RsimP{i} G_2, v_2$ and thus, by the induction hypothesis, $f_1^{(i)}(v_1) = f_2^{(i)}(v_2)$.
    Now, take any $u \in V_1 \cup V_2$ and write $G$ for the graph $G_1$ or $G_2$ depending on where $u$ was chosen from.
    For $k \in \{1, 2\}$ define 
    \begin{equation*}
        U_k = \{ u_k \in V_k \mid (v_k, u_k) \in E_k \text{ and } G_k, u_k \sim^{i, c} G, u \}.
    \end{equation*}
    By our choice of vertices, $G_1, u_1 \sim^{i, c} G_2, u_2$ for all $u_1 \in U_1$
    and $u_2 \in U_2$. Thus, by the induction hypothesis, $f_1^{(i)}(u_1) = f_2^{(i)}(u_2)$.
    Notice that, by $c$-graded back-and-forth, $|U_1| = |U_2|$ or
    $|U_1|, |U_2| \geq c$. Repeating this argument for all $u \in V_1 \cup V_2$
    we see that the multisets of feature embedding for $k \in \{1, 2\}$
    \begin{equation*}
        F_k = \msl f^{(i)}_k(u_k) \mid u_k \in N_{\text{out}}^{G_k}(v_k) \msr
    \end{equation*}
    must have the same $c$-restriction; i.e. $(F_1)_{|c} = (F_2)_{|c}$.
    Thus, all $c$-bounded aggregation functions have the same output on $F_1$ and $F_2$.
    Lastly, notice that, by the global gradedness condition, we can enumerate elements $u_1^{(1)}, \dots, u_1^{(n)} \in V_1$
    and $u_2^{(1)}, \dots, u_2^{(n)} \in V_2$ in such a way that $n = |V_1| = |V_2|$ and $G_1, u_1^{(j)} \sim^{i, c} G_2, u_2^{(j)}$
    for all $j \in [1, n]$. Then, by the induction hypothesis, $f_1^{(i)}(u_1^{(j)}) = f_2^{(i)}(u_2^{(j)})$.
    Recalling that
    \begin{align*}
        f^{(i+1)}(v) = \comb_i\Big(&f^{(i)}(v), \\&
            \agg_i\big(\msl f^{(i)}(u) \mid u \in N_{\text{out}}^G(v) \msr\big),\\&
            \rea_i\big(\msl f^{(i)}(u) \mid u \in V \msr \big)
            \Big)
    \end{align*}
    we conclude that $f_1^{(i+1)}(v_1) = f_2^{(i+1)}(v_2)$ as required.
\end{proof}

Returning to Proposition~\ref{proposition_hard_part_of_characterisation_weak},
we suppose that $\varphi$ is an FO classifier captured by an ACR-GNN $\gnn$ with $c$-bounded aggregation functions.
Writing $L$ for the number of layers in $\gnn$ we have, by Lemma~\ref{lemma_bisim_invar_acr}, that
$\gnn$ must be invariant under $\Rsim$; and thus so must $\varphi$. For the rest of the section let us keep the number of layers $L$ fixed.

In the sequel,  for a given pointed graph $G, v$, we will construct
various pointed graphs $G', v$ over the same set of vertices and features.
When doing so we will be careful so as to not break $G, v \Rsim G', v$,
which ensures that  $G, v \models \varphi$ if and only if $G', v \models \varphi$.
To help us achieve this, we will show Lemmas~\ref{prop_free_edge_transfer}~and~\ref{prop_free_witness},
the proofs of which require the following definition:
\begin{definition}
    A $\sim^{L, c}$-type of a vertex $v_1$ in $G_1$ is the set
    of all pairs $G_2, v_2$ such that $G_1, v_1 \sim^{L, c} G_2, v_2$.
    We denote the $\sim^{L, c}$-type of $v_1$ in $G_1$ by $\tp^{G_1}[v_1]$.
\end{definition}
We will sometimes abuse notation and write $\lambda \in \sim^{L, c}$ to say that
$\lambda$ ranges over $\sim^{L, c}$-types.

\propfreeedgetransfer*
\begin{proof}
    We proceed by induction on $i \leq L$ with the hypothesis that
    $\tpP{i}^G[u] = \tpP{i}^{G'}[u]$ for all $u \in V$. The base case $i = 0$
    is trivial as the feature vectors of every $u \in V$ coincide in both graphs.
    For the inductive step $i{+}1$ take any $u \in V$ and consider $(u, u') \in E$.
    In the case where $(u, u') \in E'$ there is nothing to do as $\tpP{i}^G[u'] = \tpP{i}^{G'}[u']$ by I.H.
    On the other hand, if $(u, u') \not\in E'$, then $u = v$ and $u' = w$ as this is the only edge
    removed when forming $G'$. But notice that $(v, w') \not\in E$, whilst $(v, w') \in E'$.
    By our inductive hypothesis $\tpP{i}^G[w] = \tpP{i}^{G'}[w]$ and $\tpP{i}^G[w'] = \tpP{i}^{G'}[w']$,
    whilst by the prerequisites of our lemma $\tpP{i}^G[w] = \tpP{i}^{G}[w']$.
    Thus, $\tpP{i}^G[w] = \tpP{i}^{G'}[w']$. Noting that the case where $(u, u') \not\in E$
    but $(u, u') \in E'$ is symmetric, it is easy to see that $\tpP{i+1}^G[u] = \tpP{i+1}^{G'}[u]$
    as required.
\end{proof}

\propfreewitness*
\begin{proof}
    We again consider induction on $i \leq L$ with the hypothesis that $\tpP{i}^G[u] = \tpP{i}^{G'}[u]$ for all $u \in V$.
    The base case $i = 0$ is again immediate.
    Thus, for step $i{+}1$, take any $u, u' \in V$. Suppose, first, that $(u, u') \in E$ and $(u, u') \in E'$.
    Then, by I.H., $\tpP{i}^G[u'] = \tpP{i}^{G'}[u']$.
    On the other hand, suppose that $(u, u') \not\in E$ but $(u, u') \in E'$.
    Then, $u = v$ and $u' = w'$. But $(v, w^{(i)}) \in E'$ for $i \in [1, c]$ and,
    by I.H. $\tpP{i}^G[w^{(i)}] = \tpP{i}^{G'}[w^{(i)}]$ and $\tpP{i}^G[w'] = \tpP{i}^{G'}[w']$.
    Thus, by the prerequisites of the lemma, $\tpP{i}^{G}[w^{(i)}] = \tpP{i}^{G'}[w']$.
    Since graded back-and-forth is limited to $c$ elements, we have that the extra successor $w'$ makes no meaningful contribution.
    Thus, we have $\tpP{i+1}^G[u] = \tpP{i+1}^{G'}[u]$ as required.
\end{proof}

For the rest of the section we associate with each pointed graph $(G, v)$ a
function $n^G_v \colon V \to \mathbb{N}$ such that:
\begin{itemize}
    \item $n^G_v(v) = 1$ and
    \item $\{n^G_v(u) \mid u \in V, \tp^G[u] = \lambda \} = [1, |\{ u \in V \mid \tp^G[u] = \lambda \}|]$
    for each $\sim^{L, c}$-type $\lambda$.
\end{itemize}
In other words, $n^G_v$ is some enumeration of elements sharing the same $\sim^{L, c}$-type.
This enumeration will help us find ``canonical'' witnesses for graded requirements in the constructions to follow.
When the vertex $v$ is clear from context or not important, we simply write $n^G$ for $n^G_v$.
In the sequel, when given a $\sim^{L,c}$-type $\lambda$ and a graph $G$ over vertices $V$,
we will write $V^\lambda$ for the set of all $v \in V$ having $\tp^G[v] = \lambda$.
Recalling that $N^G_\text{out}(v)$ is the out-neighborhood of $v$, we will write
$N^{G, \lambda}_\text{out}(v)$ for the set $N^G_\text{out}(v) \cap V^\lambda$.

With the mechanisms in Lemmas~\ref{prop_free_edge_transfer}~and~\ref{prop_free_witness}
we will now be able to define the promised companion structures.
We begin with a construction of a partial companion graph:
\begin{lemma}\label{lemma_initial_good_graph}
    Take any featured graph $G$. Then, there is a graph $G^\star$
    over the same set of vertices and features such that:
    \begin{enumerate}
        \item $G, v \Rsim G^\star, v$;
        \item if $n^G(v) = 1$ and $(v, w) \in E^\star$, then for each $w' \in V$ with $\tp^G[w'] = \tp^G[w]$
        and $n^G(w') < n^G(w)$ we have $(v, w') \in E^\star$;
        \item if $n^G(v) = 1$ and $(v, w) \in E^\star$ with $n^G(w) \geq c$,
        then $(v, w') \in E^\star$ for all $w' \in V$ having $\tp^G[w'] = \tp^G[w]$;
    \end{enumerate}
    for all $v, w \in V$.
\end{lemma}
\begin{proof}
    Let us fix any $v \in V$ having $n^G(v) = 1$ and
    let $G'$ be the graph $G$ but with $E' := E \setminus \{ (v, u) \mid u \in V \}$.
    That is to say, $G'$ is the graph $G$ but with outgoing edges of $v$ removed.
    We complete $G'$ so that it satisfies conditions 1--3 with respect to $v$ as follows.
    Taking any $\lambda \in \sim^{L, c}$ recall that $N_\text{out}^{G,\lambda}(v)$
    is the set of all $w \in V$ having $\tp^G[w] = \lambda$ with $(v, w) \in E$.
    Depending on the size of $N_\text{out}^{G,\lambda}(v)$ the following operations are conducted.
    If $|N_\text{out}^{G,\lambda}(v)| \geq c$, then
    add $(v, w')$ to $E'$ for all $w' \in V$ having $\tp^G[w'] = \lambda$.
    Elsewise, add $(v, w')$ to $E'$ for all $w' \in V$ having $\tp^G[w'] = \lambda$ and $n^G(w') \leq |N_\text{out}^{G,\lambda}(v)|$.
    We claim that repeating the above for each $\lambda \in \sim^{L, c}$
    results in the required $G'$.
    Conditions 2 and 3 are easy to verify, thus let us fixated on 1.
    Notice that the operations carried out above can be viewed as applications of
    Lemmas~\ref{prop_free_edge_transfer}~and~\ref{prop_free_witness}.
    Combined with the fact that only outgoing edges of $v$ were altered, $G, u \Rsim G', u$ for all $u \in V$ as required.

    Now, repeat the above procedure on $G'$ and $v' \in V \setminus \{ v \}$, and subsequent graphs and vertices
    that follow. Since the outgoing edges of $v$ are not altered, and each vertex retains the same $\sim^{L, c}$-type
    throughout the construction, it is easy to verify that the fix-point graph $G^\star$
    is as required by the lemma.
\end{proof}

The full companion structure is obtained by allowing all $u' \in V^\lambda$
to mimic the behaviour of $u \in V^\lambda$ with $n^G_v(u) = 1$.
We define the operation formally and show that it does not break $\Rsim$-bisimilarity with regards to the original graph:
\begin{lemma}\label{lemma_saturation}
    Take any featured graph $G$.
    Then, there is a graph $\hG$ over the same set of vertices such that:
    \begin{enumerate}
        \item $G, v \Rsim \hG, v$;
        \item if $\tp^G[v] {=} \tp^G[u]$, then $(v, w) \in \widehat{E} \Leftrightarrow (u, w) \in \widehat{E}$;
        \item if $(v, w) \in \widehat{E}$, then for each $w' \in V$ with $\tp^G[w'] = \tp^G[w]$
        and $n^G(w') < n^G(w)$ we have $(v, w') \in \widehat{E}$;
        \item if $(v, w) \in \widehat{E}$ and $n^G(w) \geq c$,
        then $(v, w') \in \widehat{E}$ for all $w' \in V$ having $\tp^G[w'] = \tp^G[w]$;
    \end{enumerate}
    for all $u, v, w \in V$.
\end{lemma}
\begin{proof}
    Let $G^\star$ be the graph obtained from $G$ by applying Lemma~\ref{lemma_initial_good_graph}.
    Notice that Conditions 1, 3, and 4 are already met for $v \in V$ having $n^G(v) = 1$.
    Fixing some such $v \in V$ let us take any $u \in V$ with $\tp^G[u] = \tp^G[v]$ and $n^G(u) > 1$.
    We define $G'$ to be $G^\star$ but with
    $E' := \big(E^\star \setminus \{ (u, w) \mid w \in V  \}\big) \cup \{ (u, w) \mid w \in V, (v, w) \in E \}$.
    That is to say, $E'$ is the set of edges $E^\star$ but with the outgoing edges of $u$
    replaced with ``copies'' of outgoing edges $v$. It should be easy to see that Conditions 2--4
    are met with respect to $v, u$ in $G'$. To see that condition 1 holds notice that
    in $G$ we had that the sets $\{ w \in V \mid (v, w) \in E, \tpmo[w] = \kappa \}$
    and $\{ w' \in V \mid (v, w') \in E, \tpmo[w'] = \kappa \}$ are of the same cardinality or both greater than $c$
    for any choice of $\kappa \in \sim^{L-1, c}$. Thus, the operations performed on $G'$
    are nothing more but those in Lemmas~\ref{prop_free_edge_transfer}~and~\ref{prop_free_witness}.
    Thus, $G, u' \Rsim \hG, u'$ for all $u' \in V$ as required.

    Now, repeat the above procedure for each $u \in V$ with $n^G(u) > 1$ and $\tp^G[u] = \tp^G[v]$, and all
    $v \in V$ having $n^G(v) = 1$. The resulting graph $\widehat{G}$ then satisfies the requirements
    of the lemma. 
\end{proof}

Having the above we are now able to provide each pointed graph $G_1, v_1$
with a companion graph $\hG_1$ such that $G_1, v_1 \Rsim \hG_1, v_1$, and
$\hG_1$ has elements that, in a sense, behaved predictively.
Now, recall the high-level proof strategy of Proposition~\ref{proposition_hard_part_of_characterisation}:
suppose $G_1, v_1 \Rsimq G_2, v_2$ are such that $G_1, v_1 \models \varphi$, and where $q' = q{+}c{-}1$
and $q$ is the quantifier rank of $\varphi$.
Construct companion graphs $\hG_1$ and $\hG_2$
such that $G_i, v_i \Rsim \hG_i, v_i$ for both $i \in \{ 1, 2 \}$, and such
that $G_1, v_1$ and $G_2, v_2$ agree on all FO classifiers of quantifier rank $q$.
This will yield that $\hG_2, v_2 \models \varphi$ and thus, $G_2, v_2 \models \varphi$.
Notice, however, that constructing $\hG_1$ and $\hG_2$ as done in Lemma~\ref{lemma_saturation}
need not yield companion structures which agree on all FO classifiers.
\begin{example}Consider the pointed graphs $G_1, v_1$ (left) and $G_2, v_2$ (right) below.
    Clearly, $G_1, v_1 \RsimPPP{1}{3}{5} G_2, v_2$ as
    $G_1, x_1 \sim^{1, 3} G_2, x_2$ for all $x \in \{v, u, w, t\}$.
    Under the right enumeration we can use Lemma~\ref{lemma_saturation} to produce 
    companion structures that satisfy $\hG_1 = G_1$ and $\hG_2 = G_2$.
    But $\hG_1, v_1$ satisfies $\exists y \exists z \big( E(x,y) \wedge E(y, z) \big)$,
    whilst $\hG_2, v_2$ does not.
    
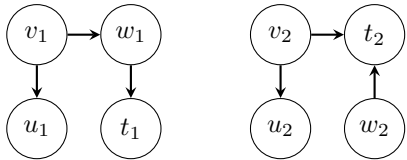
\begin{center}
    \begin{tikzpicture}[
    s_node/.style={circle, draw, minimum size=8mm},
    iota_node/.style={circle, draw, fill=green!20, minimum size=8mm},
    t_node/.style={circle, draw, fill=orange!20, minimum size=8mm},
    edge/.style={->, thick, >=stealth}
]
    \node[s_node] (v) at (0,0) {$v_1$};
    \node[s_node] (u) at (0,-1.25) {$u_1$};
    \node[s_node] (w) at (1.25,0) {$w_1$};
    \node[s_node] (t) at (1.25,-1.25) {$t_1$};

    \draw[edge] (v) -- (u);
    \draw[edge] (w) -- (t);
    \draw[edge] (v) -- (w);
\end{tikzpicture}\hspace*{10mm}
    \begin{tikzpicture}[
    s_node/.style={circle, draw, minimum size=8mm},
    iota_node/.style={circle, draw, fill=green!20, minimum size=8mm},
    t_node/.style={circle, draw, fill=orange!20, minimum size=8mm},
    edge/.style={->, thick, >=stealth}
]
    \node[s_node] (v) at (0,0) {$v_2$};
    \node[s_node] (u) at (0,-1.25) {$u_2$};
    \node[s_node] (t) at (1.25,0) {$t_2$};
    \node[s_node] (w) at (1.25,-1.25) {$w_2$};

    \draw[edge] (v) -- (u);
    \draw[edge] (w) -- (t);
    \draw[edge] (v) -- (t);
\end{tikzpicture}
    \end{center}
\end{example}
To overcome this we alter the companion graph $\hG_2$.
The construction below is a refinement Lemma~\ref{lemma_saturation}, but with $\hG_1$ and
$u_1 \in V_1$ with $n^{G_1}_{v_1}(u_1) = 1$ acting as examples for $\hG_2$ and $u_2 \in V_2$ with $\tp^{G_1}[u_1] = \tp^{G_2}[u_2]$:
\begin{lemma}\label{lemma_homogeneity}
    Suppose $G_1, v_1 \Rsimq G_2, v_2$ with $q' \geq c$.
    Now, let $\hG_1$
    be the graph obtained by applying Lemma~\ref{lemma_saturation} on $G_1$.
    Then there is a graph $\hG_2$
    such that:
    \begin{enumerate}
        \item $\hG_2$ satisfies the requirements  of Lemma~\ref{lemma_saturation} with respect to $G_2$
        \item $G_1, u_1 \Rsimq G_2, u_2$ iff $\hG_1, u_1 \Rsimq \hG_2, u_2$;
        \item if $|N_{\text{out}}^{\hG_1, \lambda}(u_1)| < c$ and $\tp^{\hG_1}(u_1) = \tp^{\hG_2}(u_2)$ then
        $|N_{\text{out}}^{\hG_1, \lambda}(u_1)| = |N_{\text{out}}^{\hG_2, \lambda}(u_2)|$;
        \item if $|N_{\text{out}}^{\hG_1, \lambda}(u_1)| \geq c$ and $\tp^{\hG_1}(u_1) = \tp^{\hG_2}(u_2)$ then
        $N_{\text{out}}^{\hG_2, \lambda}(u_2) = V_2^\lambda$;
    \end{enumerate}
    for each $u_1 \in V_1$, $u_2 \in V_2$ and $\lambda \in \sim^{L, c}$.
\end{lemma}
\begin{proof}
    We begin by providing an analogue of Lemma~\ref{lemma_initial_good_graph}
    for $G_2$.
    For this purpose, fix some $u_2 \in V_2$
    having $n^{G_2}(u_2) = 1$
    and let $G_2'$ be the graph $G_2$ but with $E_2' = E_2 \setminus \{ (u_2, w_2) \mid w_2 \in V_2 \}$.
    We complete the definition of $G_2'$ by consulting $\hG_1$.
    For this, purpose, find some $u_1 \in V_1$ such that $\hG_1, u_1 \sim^{L, c} G_2, u_2$.
    This is guaranteed to exist by the global gradedness condition of $G_1, v_1 \Rsimq G_2, v_2$
    and the fact that $G_1, u_1 \Rsim \hG_1, u_1$. Now, pick any $\lambda \in \sim^{L, c}$
    and let $m = |N_{\text{out}}^{\hG_1, \lambda}(u_1)|$.
    Notice that this implies that $m \leq |V_1^\lambda|$.
    If $m \geq q'$, we have, by global gradedness of $G_1, v_1 \Rsimq G_2, v_2$, that $q' \leq |V_2^\lambda|$.
    With that we are able to add edges $(u_2, w_2)$ to $E'$ for each $w_2 \in V_2$ having $\tp^{G_2}[w_2] = \lambda$.
    This gives us $|N_{\text{out}}^{G_2', \lambda}(u_2)| \geq q' \geq c$
    thus securing condition 4 for $u_2$ and $\lambda$ in $G'$.
    On the other hand, if $m < q'$, we have, again by global gradedness, that $m \leq |V_2^\lambda|$.
    If $m \geq c$, then proceed as in the previous case. Otherwise,
    add $(u_2, w_2)$ to $E'$ for each $w_2 \in V_2$ having $\tp^{G_2}[w_2] = \lambda$ and $n^G(w_2) \leq m$.
    This then satisfies condition 3 for $u_2$ and $\lambda$ in $G'$.
    
    Now, repeat the above for each $\lambda \in \sim^{L, c}$.
    It is easy to verify that conditions 3 and 4 are met for $u_2$ in $G'$.
    Condition 2 is established by again noting that the sets
    $\{ w_1 \in V_1 \mid (u_1, w_1) \in \widehat{E}_1, \tpmo^{\hG_1}[w_1] = \kappa \}$ and
    $\{ w_2 \in V_2 \mid (u_2, w_2) \in E_2, \tpmo^{G_2}[w_2] = \kappa \}$ are of the same cardinality or both greater than $c$
    for any choice of $\kappa \in \sim^{L-1, c}$. Thus, the operations performed on $G_2'$
    are those in Lemmas~\ref{prop_free_edge_transfer}~and~\ref{prop_free_witness}.
    By repeating the procedure on $G'$ (and graphs thereafter) and other elements $u_2 \in V_2$
    having $n^{G_2}(u_2) = 1$, we reach a graph $G^\star_2$ complying with Lemma~\ref{lemma_initial_good_graph}.
    By applying Lemma~\ref{lemma_saturation} on $G^\star_2$ we obtain $\hG_2$
    for which condition 1--4 trivially hold.
    %
\end{proof}

The following is then the final piece of Proposition~\ref{proposition_hard_part_of_characterisation_weak}:
\begin{lemma}\label{lemma_upgrading}
    Suppose $\varphi$ is an FO classifier that is invariant under $\Rsim$.
    Then, it is also invariant under $\Rsimq$, where $q' := q{+}c{-}1$, and $q$ is the quantifier rank of $\varphi$.
\end{lemma}
\begin{proof}
    Let $G_1, v_1 \Rsimq G_2, v_2$ and suppose that $G_1, v_1 \models \varphi$.
    We will show that $G_2, v_2$ must then satisfy $\varphi$ as well.

    Let $\hG_1, v_1$ be the graph obtained by applying the construction in Lemma~\ref{lemma_saturation}
    on $G_1, v_1$.
    Similarly, let $\hG_2, v_2$ be the graph obtained form $G_1, v_1$, $\hG_1, v_1$ and $G_2, v_2$
    as described in Lemma~\ref{lemma_homogeneity}.
    Clearly, $G_1, v_1 \Rsim \hG_1, v_1$ and $G_2, v_2 \Rsim \hG_2, v_2$.
    Since $\varphi$ is invariant under $\Rsim$, we have that
    $\hG_1, v_1 \models \varphi$.
    We first show $\hG_2, v_2 \models \varphi$ by arguing that $\hG_1, v_1$ and $\hG_2, v_2$ agree on all
    FO classifiers of quantifier rank $q$.

    The above is established by providing a winning strategy for the duplicator in the $q$-round (not $q' = q{+}c{-}1$ round)
    Ehrenfeucht–Fra\"iss\'e games on $\hG_1, v_1$ and $\hG_2, v_2$ (see \cite[Section~3]{libkin_fmt} for a proper introduction into
    the games semantics).
    Given $\hG_1, \bar{a}$ and $\hG_2, \bar{b}$ the duplicator will maintain the
    following for all $i, j \in [1, |\bar{a}|]$:
    \begin{itemize}
        \item $a_i = a_j \Leftrightarrow b_i = b_j$;
        \item $\tp^{\hG_1}[a_i] = \tp^{\hG_2}[b_i]$;
        \item $n^{G_1}(a_i) = n^{G_2}(b_i)$ or $n^{G_1}(a_i), n^{G_2}(b_i) \geq c$;
        \item $\hG_1 \restriction_{\bar{a}} \simeq \hG_2 \restriction_{\bar{b}}$
        (i.e. there is an isomorphism between the restriction of $\hG_1$ to $\bar{a}$ and the restriction of $\hG_2$ to $\bar{b}$).
    \end{itemize}
    Since $\tp^{\hG_1}[v_1] = \tp^{\hG_2}[v_2]$ and, by definition, $n^{G_1}(v_1) = n^{G_2}(v_2) = 1$,
    it is easy to see that the base case ($\hG_1, v_1$ and $\hG_2, v_2$) holds.
    Suppose we find ourselves at round $k < q$ having reached positions $\hG_1, \bar{a}$ and $\hG_2, \bar{b}$.
    Without loss of generality, let the spoiler pick some vertex $u_1 \in V_1$ thus forming the pointed structure
    $\hG_1, \bar{a}u_1$. In case $u_1 = a_i$ for some $1 \leq i \leq k$, the duplicator's strategy is to answer with $u_2 = b_i$.
    If $u_1$ is not in $\bar{a}$, then the duplicators strategy is to answer with $u_2 \in V_2$
    such that
    \begin{itemize}
        \item $u_2 \not\in \bar{b}$
        \item $\tp^{\hG_1}[u_1] = \tp^{\hG_2}[u_2]$
        \item if $n^{G_1}(u_1) < c$, then $n^{G_1}(u_1) = n^{G_2}(u_2)$;
        \item if $n^{G_1}(u_1) \geq c$, then $n^{G_2}(u_2) \geq c$.
    \end{itemize}
    We show that the above conditions can be met.
    If $n^{G_1}(u_1) < c$, then by the global requirements of $G_1, v_1 \Rsimq G_2, v_2$
    we have that there is some $u_2 \in V_2$ such that $\tp^{G_1}[u_1] = \tp^{G_2}[u_2]$
    and $n^{G_1}(u_1) = n^{G_2}(u_2)$.
    Clearly, by Lemmas~\ref{lemma_saturation}~and~\ref{lemma_homogeneity}, $\tp^{\hG_1}[u_1] = \tp^{\hG_2}[u_2]$.
    Now, suppose for contradiction that $u_2 \in \bar{b}$; say $b_i = u_2$. But then, by I.H., we have
    that $n^{G_1}(a_i) = n^{G_2}(b_i)$ and $\tp^{\hG_1}[a_i] = \tp^{\hG_2}[b_i]$.
    Thus, $a_i = u_1$ contradicting our initial assumption that $u_1 \not\in \bar{a}$.
    In the case where $n^{G_1}(u_1) \geq c$ proceed as follows.
    Write $\lambda = \tp^{\hG_1}[u_1]$. Since $(\hG_1, v_1) \Rsimq (\hG_2, v_2)$ we have that
    the sets $U_1^{\lambda} = \{ u \in V_1 \mid \tp^{G_1}[u] = \lambda, n^{G_1}(u) \geq c \}$ and
    $U_2^{\lambda} = \{ u \in V_2 \mid \tp^{G_2}[u] = \lambda,  n^{G_2}(u) \geq c  \}$
    are either equal in cardinality or of size at least $q'{-}(c{-}1) = q$.
    Notice that, by Lemmas~\ref{lemma_saturation}~and~\ref{lemma_homogeneity}, the aforementioned sets
    can be equivalently defined by taking $\sim^{L, c}$-types over $\hG_1$ and $\hG_2$;
    that is to say $U_1^{\lambda} = \{ u \in V_1 \mid \tp^{\hG_1}[u] = \lambda,  n^{G_1}(u) \geq c  \}$ and similarly for $U_2^{\lambda}$.
    Now, let $m$ be the number of elements $a$ amongst $\bar{a}$ that have $\tp^{\hG_1}(a) = \lambda$ and $n^{G_1}(a) \geq c$.
    Clearly, $m \leq k < q$. Since $u_1 \not\in \bar{a}$ we have that $|U_1^{\lambda}| \geq m+1$.
    Then, by definition of $\Rsimq$ and the fact that $m+1 \leq q$, we have $|U_2^{\lambda}| \geq m+1$.
    Since only $m$ elements amongst $\bar{b}$ have the
    $\sim^{L, c}$-type $\lambda$, the duplicator maintains a winning position by answering with any $u_2 \in U_2^{\lambda} \setminus \bar{b}$.
    We need now only verify that $\hG_1 \restriction_{\bar{a}u_1} \simeq \hG_2 \restriction_{\bar{b}u_2}$.
    Since we started from a position in which $\hG_1 \restriction_{\bar{a}} \simeq \hG_2 \restriction_{\bar{b}}$
    and are only concerned with sentences over vertex-labeled graphs, it is sufficient to show that
    $\hG_1 \restriction_{a_iu_1} \simeq \hG_2 \restriction_{b_iu_2}$ for each $1 \leq i \leq k$.
    Taking any $i$ in the appropriate range we see that if $a_i = u_1$, then $b_i = u_1$ and thus the required isomorphism is implied
    by the induction hypothesis.
    If $a_i \neq u_1$, then $b_i \neq u_2$ and we proceed as follows.
    Since $\lambda$ is the $\sim^{L, c}$-type of both $u_1$ and $u_2$, we have that
    $f_1(u_1) = f_2(u_2)$.
    Now, we claim that
    $(u_1, a_i) \in \widehat{E}_1 \Leftrightarrow (u_2, b_i) \in \widehat{E}_2$ and
    $(a_i, u_1) \in \widehat{E}_1 \Leftrightarrow (b_i, u_2) \in \widehat{E}_2$.
    To see the former assume that $(u_1, a_i) \in \widehat{E}_1$. By our induction hypothesis we have that
    $\tp^{\hG_1}[a_i] = \tp^{\hG_2}[b_i]$ and that $n^{G_1}(a_i) = n^{G_2}(b_i)$ when $n^{G_1}(a_i) < c$,
    whilst having $n^{G_2}(b_i) \geq c$ when $n^{G_1}(a_i) \geq c$.
    By the constructions in Lemmas~\ref{lemma_saturation}~and~\ref{lemma_homogeneity} (specifically, conditions 3 and 4 of both lemmas), we then have that
    $(u_2, b_i) \in \widehat{E}_2$.
    The other cases are similar.



    Concluding the above we have that $\hG_1, v_1$ and $\hG_2, v_2$ satisfy the same FO classifiers of quantifier rank $q$.
    Thus, $\hG_2, v_2 \models \varphi$.
    Since $\hG_2, v_2 \Rsim G_2, v_2$ and $\varphi$ is invariant under $\Rsim$, we must also have
    $G_2, v_2 \models \varphi$. It is then evident that $\varphi$ must be invariant under $\Rsimq$ as well.
\end{proof}

The remainder of our argument regarding Proposition~\ref{proposition_hard_part_of_characterisation}
is straightforward. We argue that disjunctions of characteristic formulas (defined below)
are enough to describe properties invariant under $\Rsimq$.

We begin with the characteristic formula for graded modal logic found in \cite{otto2023gradedmodallogiccounting}:
\begin{definition}
    Given $L \geq 0$, $c \geq 1$
    the $\GML$-characteristic $(L, c)$-formula $\chi_{G, v}^{(L, c)}$ of a pointed graph $G, v$
    is a formula of $\GML$ defined inductively on $L$
    as follows:
    \begin{align*}
        &\chi_{G, v}^{(0, c)} = 
            \bigwedge_{i \in [1, d]}^{f(v)[i] = 1} p_i \wedge \bigwedge_{i \in [1, d]}^{f(v)[i] = 0} \neg p_i,\\
        &\chi_{G, v}^{(L+1, c)} =
            \chi_{G, v}^{(L, c)} \wedge
            \bigwedge_{n=1}^{c}
            \bigwedge_{\lambda \in \sim^{L, c}}
            \bigg( \\
            &\qquad\bigwedge_{u \in N_\text{out}^{G,\lambda}(v)}^{|N^{G,\lambda}_\text{out}(v)| \geq n} 
            \Diamond^{\geq n} \chi_{G, u}^{(L, c)} \wedge 
            \bigwedge_{u \in N_\text{out}^{G,\lambda}(v)}^{|N^{G,\lambda}_\text{out}(v)| < n} 
            \neg \Diamond^{\geq n} \chi_{G, u}^{(L, c)}
            \bigg).
    \end{align*}
\end{definition}

Write $X^{(L, c)}$ for the set of all $\GML$-characteristic $(L, c)$-formulas
over proposition letters $p_1, \dots, p_d$. Given $\chi \in X_{(L, c)}$
we will write $V^\chi$ for the set of all $v \in V$
such that $\chi = \chi_{G, v}^{(L, c)}$.

\begin{definition}
    Given $L \geq 0$, $c \geq 1$, and $q \geq 1$
    the $\GML^\exists$-characteristic $(L, c, q)$-formula $\gamma_{G, v}^{(L, c, q)}$ of a pointed graph $G, v$
    is a formula in $\GML^\exists$ defined as follows:
    \begin{equation*}
        \displaystyle \gamma_{G, v}^{(L, c, q)} := \chi_{G, v}^{(L, c)} \wedge
        \bigwedge_{n = 1}^{q} \bigg( \bigwedge_{\chi \in X_G^{(L, c)}}^{|V^\chi| \geq n} \exists^{\geq n} \chi
        \wedge \bigwedge_{\chi \in X_G^{(L, c)}}^{|V^\chi| < n}
        \neg \exists^{\geq n} \chi \bigg).
    \end{equation*}
\end{definition}

\begin{lemma}\label{char_form_mlgc}
    Let $\gamma_{G_1, v_1}^{(L, c, q)}$ be the $\GML^\exists$-characteristic $(L, c, q)$-formula
    of some pointed graph $G_1, v_1$.
    Then $G_1, v_1 \RsimPP{L}{q} G_2, v_2$ if and only if $G_2, v_2 \models \gamma_{G_1, v_1}^{(L, c, q)}$. 
\end{lemma}
\begin{proof}
    We will rely on the well known fact \cite{otto2023gradedmodallogiccounting}
    \begin{equation}\tag{\ensuremath{\dagger}}\label{well_known_fact}
        G_1, u_1 \sim^{L, c} G_2, u_2 \Leftrightarrow G_2, u_2 \models \chi_{G_1, u_1}^{(L, q)}
    \end{equation}
    for any choice of $u_1 \in V_1$ and $u_2 \in V_2$.
    We proceed by noting some consequences of $G_2, v_2 \models \gamma_{G_1, v_1}^{(L, c, q)}$:
    \begin{enumerate}
        \item $G_2, v_2 \models \chi_{G_1, v_1}^{(L, c)}$;
        \item for each $\chi \in X^{L, c}$, the number of vertices vertices $u_2 \in V_2$
        such that $G_2, u_2^{} \models \chi$ is exactly $|V_1^\chi|$ or at least $c$.
    \end{enumerate} 
    By 1 and (\ref{well_known_fact}) we have $G_1, v_1 \sim^{L, c} G_2, v_2$.
    From 2 we see that at least $|V_1^\chi| = |V_2^\chi|$ or
    $|V_1^\chi|, |V_2^\chi| \geq c$ for any choice of $\chi \in X^{L, c}$. By (\ref{well_known_fact})
    it is then immediate that $G_1, u_1 \sim^{L, c} G_2, u_2$ for each $u_1 \in V_1^\chi$
    and $u_2 \in V_2^\chi$. It is then easy to verify that $G_1, v_1 \RsimPP{L}{q} G_2, v_2$.

    On the other hand suppose that $G_1, v_1 \RsimPP{L}{q} G_2, v_2$.
    Then, using (\ref{well_known_fact}), we have $G_2, v_2 \models \chi_{G_1, v_1}^{(L, c)}$. Additionally, for each $\chi \in X^{(L, c)}$,
    we have $|V_1^\chi| = |V_2^\chi|$ or $|V_1^\chi|, |V_2^\chi| \geq c$.
    It is then easy to see that $G_2, u_2 \models \gamma_{G_1, v_1}^{(L, c, q)}$.
\end{proof}

Let $\mathcal{P}$ be some property concerning featured pointed graphs.
In the following lemma we will identify $\mathcal{P}$ as the set of all pointed graphs
that satisfy property $\mathcal{P}$.

\begin{lemma}\label{lemma_sim_formula}
    Suppose that a property $\mathcal{P}$ is invariant under $\RsimPP{L}{q}$.
    Then, there is a formula $\psi$
    of $\GML^{\exists}$ such that $(G, v) \in \mathcal{P}$ if and only if
    $G, v \models \psi$.
\end{lemma}
\begin{proof}
    Let $C_1, C_2, \dots$ be the equivalence classes of $\RsimPP{L}{q}$ on graphs
    with $d$-dimensional feature vectors.
    By Lemma~\ref{char_form_mlgc} each class $C_i$ is uniquely characterised by a $\GML_{\exists}$-characteristic $(L, c,q)$-formula
    $\gamma_{C_i}^{(L, c, q)} = \gamma_{G, v}^{(L, c, q)}$, where $(G, v)$ is any member of $C_i$.
    It is easy to verify that there are finitely many $\GML_\exists$-characteristic $(L, c, q)$-formulas.
    Thus, the number of different classes must also be finite, say~$k$.
    We define the required formula as follows:
    \begin{equation*}
        \psi = \bigvee_{i \in [1, k]}^{C_i \subseteq \mathcal{P}} \gamma_{C_i}^{(L, c, q)}
    \end{equation*}
    To verify, take $G, v$ that satisfies $\psi$. Then $G, v$ satisfies the characteristic formula
    of some equivalence class $C_i \subseteq \mathcal{P}$.
    But then $(G, v)$ is in $C_i$ and thus also $\mathcal{P}$.
    For the other direction take $(G, v) \in \mathcal{P}$.
    Writing $C_i$ for the equivalence class that $(G, v)$ is part of under $\RsimPP{L}{q}$
    we have, by invariance of $\mathcal{P}$ under $\RsimPP{L}{q}$,
    that $C_i \subseteq \mathcal{P}$ and thus $G, v \models \gamma_{C_i}^{(L, c, q)}$, with
    which in turn we have $\gamma_{C_i}^{(L, c, q)} \models \psi$.
\end{proof}

Since Proposition~\ref{proposition_hard_part_of_characterisation_weak} gives us that
the FO classifier $\varphi$ mentioned in Proposition~\ref{proposition_hard_part_of_characterisation}
is invariant under $\Rsimq$, we have, by Lemma~\ref{lemma_sim_formula}, that $\varphi$ is equivalent to
a formula (i.e. classifier) in $\GML^\exists$ as required.

\end{document}